\documentclass{article} %
\usepackage{iclr2026_conference,times}
\usepackage[utf8]{inputenc}\usepackage[T1]{fontenc} 

\usepackage{amsmath,amsfonts,bm}

\def\eqref#1{equation~\ref{#1}}

\def\1{\bm{1}}

\def\eps{{\epsilon}}

\DeclareMathAlphabet{\mathsfit}{\encodingdefault}{\sfdefault}{m}{sl}
\SetMathAlphabet{\mathsfit}{bold}{\encodingdefault}{\sfdefault}{bx}{n}

\newcommand{\KL}{D_{\mathrm{KL}}}

\usepackage[hidelinks]{hyperref}
\usepackage{url}
\usepackage{subcaption}

\usepackage{tabularx}
\usepackage{algorithm}
\usepackage{algpseudocode}
\usepackage{adjustbox}
\usepackage{twemojis}

\title{TROLL: Trust Regions improve Reinforcement Learning for Large Language Models}

\author{
\textbf{Philipp Becker}$^1$\thanks{Equal contribution. Author order was decided by a fair coin flip.}~,~
\textbf{Niklas Freymuth}$^{1*}$,~ 
\textbf{Serge Thilges}$^1$,~
\textbf{Fabian Otto}$^2$,~
\textbf{Gerhard Neumann}$^1$\\
$^1$Karlsruhe Institute of Technology,~~$^2$Microsoft Research
\vspace{-12pt}
}

\newif\ifcameraready
\camerareadytrue  %

\ifcameraready
    \iclrfinalcopy 
\fi

\renewcommand{\KL}[2]{\textrm{KL}\left( {#1} \parallel {#2} \right)}

\newcommand{\old}[1]{{#1}_{\textrm{old}}}

\newcommand{\pol}{\pi_\theta (o_t \,|\, \bm{q}, \bm{o}_{<t}) }
\newcommand{\hatpol}{\hat{\pi}_\theta (o_t \,|\, \bm{q}, \bm{o}_{<t}) }
\newcommand{\oldpol}{\old{\pi} (o_t \,|\, \bm{q}, \bm{o}_{<t})}
\newcommand{\tarpol}{\tilde{\pi}_\theta (o_t \,|\, \bm{q}, \bm{o}_{<t})}

\newcommand{\advantage}{A(o_t, \bm{q}, \bm{o}_{<t})}

\newcommand{\probs}{q}
\newcommand{\oldlogits}{\old{q}^{(\textrm{log})}}
\newcommand{\tarlogits}{\tilde{q}^{(\textrm{log})}}

\newcommand{\smolurl}[1]{{\scriptsize{\url{#1}}}}

\newcommand{\revision}[1]{{#1}}

\usepackage{etoolbox}
\makeatletter
\patchcmd{\hyper@makecurrent}{%
    \ifx\Hy@param\Hy@chapterstring
        \let\Hy@param\Hy@chapapp
    \fi
}{%
    \iftoggle{inappendix}{%
        \@checkappendixparam{chapter}%
        \@checkappendixparam{section}%
        \@checkappendixparam{subsection}%
        \@checkappendixparam{subsubsection}%
        \@checkappendixparam{paragraph}%
        \@checkappendixparam{subparagraph}%
    }{}%
}{}{\errmessage{failed to patch}}

\newcommand*{\@checkappendixparam}[1]{%
    \def\@checkappendixparamtmp{#1}%
    \ifx\Hy@param\@checkappendixparamtmp
        \let\Hy@param\Hy@appendixstring
    \fi
}
\makeatletter

\newtoggle{inappendix}
\togglefalse{inappendix}

\apptocmd{\appendix}{\toggletrue{inappendix}}{}{\errmessage{failed to patch}}

\usepackage{amsthm}
\usepackage{mathtools}
\newtheorem{theorem}{Theorem}[section]

\theoremstyle{definition}
\newtheorem{definition}{Definition}[section]
\usepackage{listings}
\definecolor{codegray}{gray}{0.95}
\definecolor{keyword}{RGB}{0,102,204}
\definecolor{comment}{RGB}{0,128,0}
\usepackage[most]{tcolorbox}
\tcbuselibrary{listings,breakable}
\usepackage{xcolor}
\definecolor{midnightblue06392}{RGB}{0,63,92}

\usepackage{graphicx}
\usepackage{booktabs}
\usepackage{multirow}
\usepackage{nicefrac}

\usepackage{pgfplots}
\pgfplotsset{compat=1.18}
\usetikzlibrary{calc}
\usetikzlibrary{pgfplots.groupplots}
\usepackage[table]{xcolor}

\usepackage[acronym,nohypertypes={acronym,notation}]{glossaries}
\newacronym{rl}{RL}{Reinforcement Learning}
\newacronym{troll}{\textit{TROLL}}{Trust Region Optimization for Large Language models}
\newacronym{rlvr}{RLVR}{Reinforcement Learning from Verifiable Rewards}
\newacronym{ppo}{PPO}{Proximal Policy Optimization}
\newacronym{grpo}{GRPO}{Group-Relative Policy Optimization}
\newacronym{drgrpo}{Dr.GRPO}{GRPO Done Right}
\newacronym{dpo}{DPO}{Direct Preference Optimization}
\newacronym{gspo}{GSPO}{Group Sequence Policy Optimization}
\newacronym{trpo}{TRPO}{Trust Region Policy Optimization}
\newacronym{trpl}{TRPL}{Trust Region Projection Layers}
\newacronym{llm}{LLM}{Large Language Model}
\newacronym{kl}{KL}{Kullback-Leibler}

\newacronym{mse}{MSE}{Mean Squared Error}
\newacronym{rfpp}{RF++}{REINFORCE++}

\begin{document}

\maketitle

\begin{abstract}
Reinforcement Learning (RL) with PPO-like clip objectives has become the standard choice for reward-based fine-tuning of large language models (LLMs). 
Although recent work has explored improved estimators of advantages and normalization, the clipping mechanism itself has remained untouched.
Originally introduced as a proxy for principled KL-based trust regions, clipping is a crude approximation that often causes unstable updates and suboptimal performance. 
We replace the clip objective with a novel discrete differentiable trust region projection, which provides principled token-level KL constraints.
The projection operates on a sparse subset of the model’s most important token logits to balance computational cost and projection effectiveness. 
Our approach, Trust Region Optimization for Large Language models (TROLL), serves as a direct replacement for PPO-like clipping during training and does not alter the model’s inference behavior.
Across mathematical reasoning and code generation tasks, model families, as well as advantage-estimation methods, TROLL consistently outperforms PPO-like clipping in terms of training speed, stability, and final success rates.
\end{abstract}

\section{Introduction}
\vspace{-0.2cm}
\revision{\gls{rl}} has become the standard approach for fine-tuning and aligning \glspl{llm} with preferences or verifiable rewards.
For such post-training, the algorithms of choice are predominantly \gls{ppo}-style policy gradient approaches~\citep{schulman2017proximal}.
They first estimate an advantage function and then update the policy using an importance-weighted objective, clipped to prevent the ratio between new and old policies from deviating too much. 
Recent approaches such as \glsunset{grpo}\gls{grpo}~\citep{shao2024deepseekmath}, \glsunset{drgrpo}\gls{drgrpo}~\citep{liu2025understanding}, \glsunset{gspo}\gls{gspo}~\citep{zheng2025group}\revision{, and \glsentrylong{rfpp}~\citep{hu2025reinforcepp}} improve the estimation of advantages and normalization, resulting in significant advances in \gls{rl} for \glspl{llm}.
Yet, all these approaches rely on \gls{ppo}’s clipping-based policy update mechanism. %

Clipping is originally motivated via trust region methods~\citep{schulman2015trust,schulman2017proximal}, which provide a principled way to stabilize policy updates by constraining the \glsunset{kl}\gls{kl} divergence~\citep{kullback1951information} between successive policies during training~\citep{Kakade2002, peters2010relative}. 
While theoretically sound, implementing such trust regions is often costly, in particular with modern \glspl{llm} whose vocabularies and thus output distributions can exceed $100\,000$ entries~\citep{yang2025qwen3technicalreport,qwen2025qwen25technicalreport}.
\gls{ppo} instead clips the importance ratio to sidesteps this challenge.
While empirically successful, it is a crude approximation of the underlying trust region~\citep{wang2019trust, pmlr-v115-wang20b}.
Crucially, it may lead to unstable optimization, poorly calibrated updates, as well as sensitivity to hyperparameters and implementation details, which often culminate in suboptimal performance~\citep{Engstrom2020,andrychowicz2021what,otto2021differentiable,huang202237}.

\begin{figure}[t]
    \vspace{-1cm}
    \centering
    \hfill
    \begin{minipage}[c]{0.4\textwidth}
        \centering
        \begin{adjustbox}{trim={0.2cm} 0 {0.7cm} 0, clip} 
            \input{fig/fig1/proj_vis}
        \end{adjustbox}
    \end{minipage}\hfill
    \begin{minipage}[0.6]{0.6\textwidth}
    \input{fig/fig1/14b_Code_Math}
    \end{minipage}
    \vspace{-0.5cm}
    \caption{
    \gls{troll} overview. 
    (\textbf{Left}) Example of a 3-token distribution (cat, troll, hamster). The \textcolor{red}{old policy} favors the troll, while the \textcolor{blue}{new policy} shifts toward the hamster. 
    The \textcolor{green!50!black}{projection} ensures that the updated policy stays within the trust region (circle). 
    (\textbf{Right}) \gls{troll} yields clear performance gains over PPO-like clipping (\textit{CLIP}) on mathematical reasoning and code generation tasks, as shown for Qwen3-14B trained with GRPO.}
    \vspace{-0.7cm}
    \label{fig:figure1}
\end{figure}

As a remedy, we introduce \acrfull{troll} \twemoji{troll}, a differentiable trust region projection approach that directly enforces token-level \gls{kl} constraints between discrete distributions%
\ifcameraready%
\footnote{Project page, code, and datasets available at \url{https://niklasfreymuth.github.io/troll/}}%
\fi%
.
\gls{troll} formulates a convex optimization problem that acts as a direct replacement to \gls{ppo}-like clipping objectives.
For each token, \gls{troll} projects the output distribution of the new, updated policy onto a \gls{kl}-trust region around the old policy that was used to sample the sequence.
This process ensures that the new and old policies only differ by a given bound, preventing the policy update from diverging or collapsing. 
The left of Figure~\ref{fig:figure1} shows a $3$-dimensional example where the old policy prefers the "troll" token, the new policy leans towards the "hamster" token, and the trust region constrains the update to keep the new policy close to the old one.
The direction of the projection can be computed in closed form, while its step size is the solution to a one-dimensional convex Lagrangian dual problem. 
The projection leaves the new distribution unchanged if it already falls within the trust region, and can be efficiently solved in parallel in practice.
\gls{troll} enables differentiation through the solution of the projection problem using the OptNet framework~\citep{amos2017optnet}, which introduces only negligible computational overhead. 
This process maintains gradient information even for updates that are constrained by the trust region, in contrast to \gls{ppo}-like clipping, which cuts gradients for tokens whose ratios exceed the clipping threshold.
Further, the trust region is only effective during training and provides zero additional overhead during model inference.
To incentivize the model to stay within the trust region for successive update steps, we additionally add a simple regression term between projected and unprojected tokens. 

Applying \gls{troll} directly to \glspl{llm} is computationally infeasible, since the large token vocabulary causes prohibitively expensive projections and memory overhead.
However, natural language continuations and similarly \gls{llm} predictions are generally characterized by very few high-probability tokens~\citep{zipf1949human, piantadosi2014zipf, kunstner2024heavy, duan2024shifting, ren2024learning}. 
This property enables a sparsification scheme that discards unlikely tokens, retaining only the most relevant ones.
On average, as few as $5{-}10$ tokens generally preserve more than $99.999\%$ of the distribution's probability mass.
We thus modify \gls{troll} to handle sparse distributions, allowing it to scale to modern \glspl{llm} and act as a drop-in replacement for \gls{ppo}-style clipping.

We evaluate \gls{troll} on \gls{rlvr} problems, focusing on mathematical reasoning. 
Using \gls{troll} for GRPO~\citep{shao2024deepseekmath} with models from the Qwen3~\citep{yang2025qwen3technicalreport} and Qwen2.5~\citep{qwen2025qwen25technicalreport} families on DAPO-Math~\citep{yu2025dapoopensourcellmreinforcement} yields \revision{substantial improvements in terms of training stability and success rate for evaluations based on the number of updates as well as wall clock time.
Concretely, \gls{troll} improves roughly $3{-}10$ percentage points, or $5{-}15\%$ relative over clipping for math-based reasoning.}
To assess robustness across algorithmic variants, we experiment with PPO~\citep{schulman2017proximal}, Dr.GRPO~\citep{liu2025understanding}, GSPO~\citep{zheng2025group}, \revision{and REINFORCE++~\citep{hu2025reinforcepp}}. 
Across methods, \gls{troll} consistently enables faster learning and improves success rates, indicating that its benefits are independent of the underlying advantage estimation method. 
We demonstrate \revision{similar} improvements across additional math datasets, namely GSM8K~\citep{cobbe2021trainingverifierssolvemath} and Eurus-2-RL-Math~\citep{cui2025processreinforcementimplicitrewards}, as well as models from the LLaMA 3~\citep{grattafiori2024llama3herdmodels}, SmolLM3~\citep{bakouch2025smollm3}, and Apertus~\citep{hernández2025apertus} families.
\revision{Here, \gls{troll} enables stable training for some models where while clipping fails entirely.}
\revision{Finally, we consider code generation using the Eurus-2-RL-Code~\citep{cui2025processreinforcementimplicitrewards} dataset, showing improvements of $7{-}18$ percentage points, or $18{-}30\%$ relative over clipping across Qwen3 model sizes.}

To summarize, we
\textbf{i)} derive \gls{troll}, a fully differentiable, principled trust region projection for discrete distributions that enforces per-token \gls{kl} constraints, 
\textbf{ii)} introduce a sparsification scheme that lets the projection scale to \gls{llm}-scale vocabularies and implement it as a drop-in replacement for \gls{ppo}-style heuristic clipping across \gls{rl} algorithms,
\textbf{iii)} demonstrate through experiments spanning different advantage-estimation methods, models and datasets that \gls{troll} consistently improves success rates and training stability compared to clipping.

\section{Related Work}
\vspace{-0.5cm}
\textbf{Trust Regions in Reinforcement Learning.}
Information-theoretic trust regions based on the KL divergence~\citep{kullback1951information} are known to stabilize \revision{\gls{rl}} 
in classical~\citep{Kakade2001, Kakade2002, peters2010relative, abdolmaleki2015model, akrour2018model} as well as modern deep learning settings~\citep{schulman2015trust, schulman2017proximal}.
\revision{
\gls{trpo}~\citep{schulman2015trust} limits the KL-divergence by solving a constrained optimization problem. \gls{ppo}~\citep{schulman2017proximal} simplifies this approach to a first-order method, using a clipped surrogate objective to make policy optimization in \gls{rl} scalable~\citep{akkaya2019solving, berner2019dota, Baker2020Emergent}. 
}
However, \gls{ppo}'s trust region is less principled and very sensitive to practical implementation details~\citep{Engstrom2020, andrychowicz2021what, huang202237}.
\revision{Several methods dynamically adapt \gls{ppo}'s clipping bounds~\citep{wang2019trust, xi2025bapo}, but still require clipping.
Alternative methods ~\citep{abdolmaleki2018maximum,song2020vmpo} rely on regularizing the expected KL between subsequent policies.
In contrast, \gls{troll} enforces exact, token-wise trust regions through a differentiable KL projection, combined with adaptive strategies that control the trust region size.}

\revision{\textbf{Trust Region Projections.}}
A different line of work %
\citep{otto2021differentiable, Akrour2019} enforces trust regions using projection-based methods.
These approaches first compute the policy as usual, and then project it back into a feasible set defined by a trust region constraint.
\revision{In particular, \citet{otto2021differentiable} compute exact trust region projections for each state when using Gaussian policies, leading to improvements in high-dimensional action spaces~\citep{celik2024acquiring, li2024open, hoang2025geometry, otto2024efficient}.
Similar to us, they rely on Lagrangian optimization~\citep{boyd2004convex} and implicit differentiation~\citep{amos2017optnet}, but focus on Gaussian distributions and continuous control tasks~\citep{brockman2016openai}.}
\gls{troll} instead proposes differentiable projections for categorical distributions and provides an efficient implementation via a sparsification scheme.
This combination lets \gls{troll} scale to modern-day LLMs while preserving the stability of classical trust region methods. 

\textbf{Reinforcement Learning for LLMs.}
Recently, \gls{rl} has become a key tool in the post-training stage of \glspl{llm}.
Popular frameworks include \gls{rl} from human feedback (RLHF) \citep{christiano2017deep,ziegler2019fine, stiennon2020learning,ouyang2022training} for \gls{llm} alignment and \gls{rlvr} \citep{luong2024reft,lambert2024tulu} for reasoning tasks such as mathematical problem solving or code generation.
\revision{In RLHF, preference-based methods often optimize a reward while remaining close to a reference policy via an added expected KL penalty.
This reference policy is typically the supervised fine-tuning model~\citep{stiennon2020learning, ouyang2022training}. 
For example, \gls{dpo}~\citep{rafailov2023direct} optimizes this objective in closed form on preference data, thereby avoiding policy rollouts. 
\gls{troll} does not constrain the updates to a fixed reference, instead enforcing proximity to the previous step's policy to stabilize on-policy optimization.
While this work focuses on RLVR for \glspl{llm}, \gls{troll} can also be applied to other \gls{llm} post-training settings, and more generally, discrete \gls{rl} tasks.}

\revision{
Several recent methods avoid \gls{ppo}'s importance sampling and thus clipping entirely.
These methods combine group rewards with a standard policy gradient objective~\citep{chu2025gpg}, utilize an additional value network~\citep{richemond2024offline}, use Q-learning~\citep{clavier2025shiq}, or employ a contrastive loss to enable off-policy learning~\citep{flet2024contrastive, cohere2025command}.
However, while these approaches are promising, current research on RLVR still often relies on advantage estimation. 
While not the focus of this work, we briefly show that one such non-clipping-based approach also benefits from \gls{troll}, since its trust regions directly act on the \gls{llm}'s policy. %
}

\textbf{Advantage Estimation for LLM Post-Training.}
\gls{ppo}~\citep{schulman2017proximal} requires an expensive explicit value model.
Especially for \gls{rlvr}, where evaluating multiple rollouts per input is comparatively cheap, sample-based advantage estimation methods like \glsreset{grpo}\gls{grpo}~\citep{shao2024deepseekmath} have become popular alternatives.
Variants include %
\revision{\glsreset{drgrpo}\gls{drgrpo}~\citep{liu2025understanding}, which addresses \gls{grpo}'s optimization biases that favor longer responses, and \glsentrylong{rfpp}~\citep{hu2025reinforcepp}, which uses Global Advantage Normalization~\citep{andrychowicz2021what} to improve training stability.}
\glsreset{gspo}\gls{gspo}~\citep{zheng2025group} extends importance ratios from token to sequence level to improve update stability.
However, all these methods still depend on \gls{ppo}-style clipping to stabilize policy updates.
We introduce \gls{troll} as a more principled drop-in replacement, applicable regardless of how advantages are computed and compatible with all the approaches above.

\section{Trust Region Optimization for Large Language Models}
\glsresetall
\gls{rl} for~\glspl{llm} \revision{finetunes the LLM's parameters, $\theta$, using} policy ratio objectives~\citep{schulman2015trust} of the form
\begin{align}
    \label{eq:pg}
   \mathcal{J}_\textrm{ratio}(\theta) = \mathbb{E}_{\bm{o} \sim \old{\pi}(\bm{o} | \bm{q}) \mathcal{D}(\bm{q})} \left[ \frac{1}{|\bm{o}|} \sum_{t=1}^{|\bm{o}|} \left( \dfrac{\tarpol}{\oldpol} \advantage \right) \right],
\end{align}
where $\tarpol$ is the probability of the sampled token under the current \gls{llm} policy.
\revision{Here, $\oldpol$ is the token's probability under the \gls{llm} policy that was used for data collection in the previous iteration.}
The context sequence consists of the prompt $\bm{q}$ and prior response tokens $\bm{o}_{<t}$.
The advantage estimate $A_t = \advantage$ measures if a token is better or worse than the average behavior.
Thus, maximizing $\mathcal{J}_\textrm{ratio}(\theta)$ increases the probability of good responses while decreasing the probability of bad ones. 
In practice, we estimate $A_t$ using an explicit value model as in \glsunset{ppo}\gls{ppo}~\citep{schulman2017proximal} or purely sample-based as in \glsunset{grpo}\gls{grpo} and its variants~\citep{shao2024deepseekmath, liu2025understanding, zheng2025group}.
For such policy ratio objectives, stable and effective optimization requires keeping $\tarpol$ and $\oldpol$ close, so that the importance ratio $r_{\theta, t}=\nicefrac{\tarpol}{\oldpol}$ remains close to one~\citep{schulman2015trust, schulman2017proximal}. %
\gls{ppo} attempts to maintain this proximity by clipping the ratio,%
\begin{align}
\label{eq:ppo_clip}
\mathcal{J}_\textrm{ppo}(\theta)  = \mathbb{E}_{o_t \sim \old{\pi}(\bm{o} | \bm{q}) \mathcal{D}(\bm{q})} \left[ \frac{1}{|\bm{o}|} \sum_{t=1}^{|\bm{o}|} \min\left(r_t A_t;\mathrm{clip}\left(r_t, 1-\epsilon_\text{ppo}, 1+\epsilon_\text{ppo}\right) A_t \right) \right].
\end{align}
However, this clipping is a crude \revision{approximation of the underlying trust region principle}.
While it prevents large updates, this approach is purely heuristic and suppresses gradients when the ratio falls outside the clipping range, resulting in unstable and inefficient learning.
In contrast, token-wise \glsunset{kl}\gls{kl}-based constraints offer a principled approach to limit the change between successive policies.
Our method, \gls{troll}, implements these constraints using differentiable trust region projections~\citep{otto2021differentiable} as a drop-in replacement for the \gls{ppo}-like clipping.

\subsection{Discrete Differentiable Trust Region Projections}

\begin{figure}[t]
    \centering
    \resizebox{\textwidth}{!}{ \input{fig/fig2/overview}}
    \caption{
    \revision{\gls{troll} replaces \gls{ppo}-like clipping with a differentiable trust region projection approach. 
    Given the current output distribution of an LLM $\tarpol$, and the distribution that was used to collect the sequence $\oldpol$ for the replay buffer, 
    \gls{troll} enforces a per-token trust region by solving \autoref{eq:proj} for each token.
    The resulting distribution is then used in the RL objective (\autoref{eq:troll_obj}) to update the LLM parameters.
    To scale this approach to the vocabulary size of modern LLMs, \gls{troll} uses a sparsification approach, which allows working on a small subset of logits while retaining most of the distribution's mass. 
    }
    }
    \label{fig:figure2}
    \vspace{-0.2cm}
\end{figure}
\revision{\textbf{Trust Region Projection.}} Our projection formally solves the convex optimization problem 
\revision{\begin{align}
  \pol = & \underset{{\hatpol}}{\arg \min}  \KL{\hatpol}{\tarpol} \label{eq:proj} \\ &\text{s.t.} ~ \KL{\hatpol}{\oldpol} \leq \epsilon \nonumber
\end{align}}%
for every output token $o_t$%
\footnote{$\pol$ must also remain a valid distribution, i.e., $\sum_{o_t} \pol =1$ and $\pol\geq0$ for all $o_t$. 
We omit these constraints for brevity and elaborate in Appendix~\ref{app_sec:derivations}.}.
Intuitively, the projection finds the policy distribution closest to the current LLM policy $\tarpol$ while remaining within an $\epsilon$-bound of the policy \revision{used to collect the data for the current iteration $\oldpol$}.
The solution to this optimization problem is derived in \autoref{app_sec:troll_primal} and given as
\begin{align}
\pol \propto  \exp \left(\dfrac{\eta^* \log \oldpol + \revision{\log{}} \tarpol }{\eta^* + 1} \right) \label{eq:primal}\text{,}%
\end{align}
which is a geometric interpolation between the logits of $\tarpol$ and $\oldpol$. 
Here, $\eta^*$ acts as a step size controlling how far the projection moves the new policy to the old one. 
For each token, we can compute the optimal $\eta^*$ which enforces the trust region constraint by solving the convex dual of \autoref{eq:proj}.
This dual is a scalar optimization problem, which we derive and state in \autoref{app_sec:troll_dual}, and can be solved with sufficient accuracy using a few iterations of ternary, or more generally, $n$-ary, bracketing.
Furthermore, projecting is only necessary if the trust region bound is violated, which is only the case for very few, but highly relevant tokens.
Thus, we can avoid it for the vast majority of tokens by filtering them beforehand.

\revision{\textbf{Policy Updates with Trust Region Projections.}}
After projection, the policy $\pol$ satisfies the trust region constraint and can be \revision{used to optimize the LLM parameters $\theta$, via an update similar to} \autoref{eq:pg}.
However, the \gls{llm} output $\tarpol$ may still deviate arbitrarily from the old policy, complicating inference and successive updates. %
\revision{To avoid this,} we follow~\citet{otto2021differentiable} and address this by regressing the \gls{llm} output $\tarpol$ toward its projection $\pol$, resulting in an objective 
\revision{
\begin{align}
&\mathcal{J}_\textrm{Troll}(\theta) =\label{eq:troll_obj} \\ 
&\mathbb{E}_{o_t \sim \old{\pi}(\bm{o} | \bm{q}) \mathcal{D}(\bm{q})} \!\left[ \frac{1}{|\bm{o}|} \sum_{t=1}^{|\bm{o}|} \left( \dfrac{\pol}{\oldpol} A_t\right) - \alpha \KL{\tarpol}{\bigl\lfloor \pol \bigr\rfloor} \right] \nonumber,
\end{align}
}%
where $\lfloor~~\rfloor$ denotes \revision{the stop gradient operator} and $\alpha$ is a user-specified regression weight. 
Crucially, the projected policy $\pol$ is used to compute the ratios and as a regression target for the \gls{llm} output $\tarpol$. 
For the regression, we \revision{stop the gradients through $\pol$} so that the \gls{llm} policy $\tarpol$ is pulled towards the output of the projection $\pol$, not the other way around.
The regression term only affects projected tokens and still allows policy updates up to the \gls{kl} bound, making the approach robust to the choice of $\alpha$.
We thus set to $\alpha=1$ in all experiments for simplicity.
Notably, our objective in \autoref{eq:troll_obj} makes no assumption on the advantages $A_t$. 
Thus, \gls{troll} can be directly applied to a variety of existing advantage estimation methods, including \gls{ppo}, \gls{grpo}, \glsunset{drgrpo}\gls{drgrpo}, and \glsunset{gspo}\gls{gspo}.
\revision{\autoref{algo:troll_pseudo} provides pseudocode.}

\begin{algorithm}[t]
\caption{\revision{Optimizing LLMs with TROLL}}
\small{
\begin{algorithmic}[1]
    \State LLM policy $\tarpol$, Training data $\mathcal{D}$, KL Bound $\epsilon$
    \For{step $ s{=}1\dots $}
       \State Sample batch of questions $\bm{q}{\sim}\mathcal{D}$ 
       \State Sample responses $\lbrace\bm{o} \sim \tarpol$ using the LLM policy.
       \State Set reference policy $\pi_{\text{old}}{=}\tilde{\pi}_\theta$ 
       \State Sparsify and save corresponding logits $\oldpol$ 
       \For{minibatch $b$}
            \State Compute and sparsify current logits $\tarpol$ 
            \State Estimate advantages $A_t$ using any advantage estimation method 
            \For{response tokens $o$ in parallel}
                \If {$\KL{\tarpol}{\oldpol} \leq \epsilon$}
                    \State Set $\pol = \tarpol$ \Comment{no projection needed} 
                \Else
                    \State Compute $\eta^*$ by numerically optimizing \autoref{eq:final_dual}.
                    \State Compute $\pol$ using \autoref{eq:primal} \Comment{project to $\oldpol$}
                 \EndIf
            \EndFor
            \State Update LLM parameters $\theta$ using \autoref{eq:troll_obj}
        \EndFor
    \EndFor
\end{algorithmic}
}

\label{algo:troll_pseudo}
\end{algorithm}
\revision{\textbf{Making Trust Region Projections Differentiable.}}
To propagate gradients through our projection, we can rely on autograd tools such as PyTorch~\citep{paszke2019pytorch}, except for the numerical optimization of the dual.
Formally, the optimal $\eta^*$ is a function of the \gls{llm} policy $\tarpol$. 
To obtain a fully differentiable projection, we need the gradient $\nicefrac{\partial \eta^*}{\partial \tarpol}$, which describes how the \gls{llm} output influences the optimal step size. 
We follow the OptNet framework~\citep{amos2017optnet} and differentiate the KKT conditions~\citep{karush1939minima, boyd2004convex} of the dual solution via implicit differentiation~\citep{dontchev2009implicit} and differential calculus~\citep{magnus1989matrix}.
This approach lets us compute a closed-form gradient instead of differentiating through the numerical optimization. 
\autoref{app_subsec:proj_gradient} provides derivations.
\revision{\autoref{app_sec:code} provides 
code snippets and a schematic of our projection's compute graph, including gradients.}

\subsection{Sparse and Efficient Representations of Token Distributions}

Naively implementing \gls{troll} requires storing and projecting the full vocabulary distribution for each token. 
Using Qwen3's tokenizer ~\citep{qwen2025qwen25technicalreport} as an example, this results in an overhead of $151\,936$ logits per token, which is prohibitively expensive.
To address this issue, we sparsify both the distributions and the implementation of the projection.
We greedily select the $K$ tokens with the largest probability mass, sort them by their mass, and then only retain those needed to cover a cumulative mass of $1{-}\delta$.  
We additionally always keep the token actually selected by the LLM policy to ensure gradient information for this token. 
The top-$K$ filtering both upper bounds the number of kept logits, acting as a fail-safe to prevent excessive memory usage for high-entropy predictions, and allows for efficient sorting of relevant tokens.
Since pre-trained \glspl{llm} generally have very low perplexity~\citep{kaplan2020scaling, hoffmann2022training, ruan2024observational}, this thresholding allows us to maintain almost all of the probability mass of the logit distribution with very few average kept logits.
Empirically, using $K{=}64$ and $\delta{=}10^{-5}$ usually allows us to keep $99.999\%$ of probability mass with only $5{-}10$ average tokens for most of the tested model and task combinations.
Finally, for the discarded tokens, we cannot assume a probability of truly $0$ but have to use a small default mass $p_d{>}0$ to maintain well-behaved distributions.
After sparsification, we re-normalize the kept tokens with \autoref{app_eq:sparse_default_def}, taking into account the number of non-kept tokens and default mass. 
We perform the sparsification in chunks of the full generated sequences to prevent memory spikes.

While greedily keeping the highest-probability tokens is intuitively useful, we additionally show in \autoref{thm:greater_logit} in \autoref{app_ssec:sparsification_theorems} that it yields best possible KL approximation under mild assumptions.
Additionally, under moderate assumptions, the error introduced by sparsification is bound by
\begin{equation}
        \KL{p}{q} \le \gamma^{-1} \KL{p'}{q'} + \delta \log \frac{ \delta }{ q_{\min} },
\vspace{-0.1cm}
\end{equation}
where $p$ and $q$ are arbitrary categorical distributions, $p', q'$ the corresponding sparsified distributions, $q_{\min} \le q(x_i)$ denotes a reference lower bound and $\gamma{\approx}1$ the renormalization constant.
\autoref{thm:topk_bound} provides the proof and demonstrates that, for the hyperparameters used in Qwen3, the error incurred by enforcing the trust region on the sparsified distributions rather than on the full distributions is approximately two orders of magnitude smaller than the bound itself.
The sparsification reduces memory and computation cost to the point where \gls{troll} only incurs minimal overhead on modern \glspl{llm}, making it a practically viable alternative to \gls{ppo}-like clipping.
Further, this overhead is constant in model size, causing its relative cost to diminish for larger models.
\autoref{fig:figure2} shows a schematic overview of \gls{troll}'s training, and 
\autoref{ssec:analysis} provides additional analysis of the sparsification and projection behavior in practice.

\begin{figure}[t]
    \input{fig/result_plots/dapo_qwen3_model_size/grouped2}
    \input{fig/result_plots/eurus_code_qwen3_model_size/grouped2_with_legend}
    \vspace{-0.18cm}
    \caption{
    Comparison of \textit{TROLL} (full lines) and \textit{Clip} (dashed lines) across \gls{grpo}-trained Qwen3 models with $600$M to $14$B parameters \revision{for \texttt{DAPO} (\textbf{top}) and \texttt{Eurus-Code} (\textbf{bottom})}. 
    Full-opacity lines mark smoothed results, while the background shows original values. 
    \textit{TROLL} consistently boosts training efficiency and final success rates \revision{on math-related questions and code generation tasks. These gains translate to different evaluation datasets.}%
    }
    \label{fig:qwen3_size}
    \vspace{-0.4cm}
\end{figure}

\section{Experiments}
\label{sec:experiments}

\textbf{Datasets.}
We evaluate \gls{troll} by finetuning \glspl{llm} for mathematical reasoning \revision{and code generation} using an \gls{rlvr} setup.
DAPO-Math~\citep{yu2025dapoopensourcellmreinforcement} consists of $17$ thousand math questions and answers that are obtained from web scraping and manual annotation.
\autoref{app_sec:example_generations} provides an example question.
We randomly split off $1024$ samples to provide an in-distribution evaluation dataset, and use the remaining samples for training.
We refer to those sets as \texttt{DAPO}-Eval and \texttt{DAPO}-Train, respectively. 
Additionally, we follow the evaluation setup of \citet{cui2025entropy} and use a suite of test datasets, which we call \texttt{Math}-Eval, comprised of MATH500~\citep{hendrycks2021measuring}, AMC, AIME2024~\citep{li2024numinamath}, AIME 2025, OMNI-MATH~\citep{gao2025omni}, OlympiadBench~\citep{he2024olympiadbench}, and Minerva~\citep{lewkowycz2022solving}.
As in previous work~\citep{cui2025entropy}, we report the mean of $32$ rollouts for the comparatively small AIME2024, AIME2025, and AMC datasets to reduce evaluation variance.
\texttt{GSM8K}~\citep{cobbe2021training} contains grade school math problems with final integer answers, consisting of $8.5$k training and $1.3$k test problems.
\revision{
Finally, we consider the \texttt{Eurus}-2-RL data~\citep{cui2025processreinforcementimplicitrewards}, which contains mathematical reasoning and code generation tasks.
We consider those separately, resulting in \texttt{Eurus-Math} and \texttt{Eurus-Code}.
For both datasets, we use the train and validation sets as is.}

\revision{The three mathematical reasoning} tasks range from comparatively simple grade school problems to complex math olympiad tasks, \revision{and the code generation demands understanding of techniques such as dynamic programming, data structures and amortized analysis}.
\autoref{app_ssec:dataset_results} provides further details on all datasets \revision{and the reward calculations.}

\textbf{Models.}
We experiment with Qwen3-\{$0.6$B, $1.7$B, $4$B, $8$B, $14$B\}~\citep{yang2025qwen3technicalreport}, which we use in thinking mode, and Qwen2.5-\{$0.5$B,$1.5$B,$3$B,$7$B\}-Instruct~\citep{qwen2025qwen25technicalreport}.
Furthermore, we include both the instruct and non-instruct versions of Llama-3.1-$8$B, Llama-3.2-$3$B~\citep{grattafiori2024llama3herdmodels}, and Apertus-$8$B~\citep{hernández2025apertus}.
Finally, we include Smol-LM3-$3$B~\cite{bakouch2025smollm3} and a version of Llama fine-tuned on FineMath~\citep{finemathllama2025}.
These models range from $500$M to $14$B parameters and cover different vocabulary sizes, tokenizers, model architectures, pre-training paradigms, and datasets, as well as initial math capabilities.

\textbf{Methods.}
We focus on \gls{grpo}~\citep{shao2024deepseekmath} due to its popularity and empirical success.  
We also include \gls{ppo}~\citep{schulman2017proximal}, which uses an explicit value model for Generalized Advantage Estimation~\citep{Schulman2015a}, and \revision{three} additional \gls{grpo} variants, namely \glsunset{drgrpo}\gls{drgrpo}~\citep{liu2025understanding}, \gls{gspo}~\citep{zheng2025group}, \revision{and \glsreset{rfpp}\gls{rfpp}~\citep{hu2025reinforcepp}}.  
All methods differ in their advantage estimation and in their normalization of %
\autoref{eq:pg}, yet they all rely on \gls{ppo}-like clipping, which makes them amenable to using \gls{troll}.
We compare each method's original clipping-based version to using \gls{troll} projections, denoted as (\textit{Clip})  and (\textit{\gls{troll}}), respectively. 
\revision{Finally, we compare to BAPO~\citep{xi2025bapo} as an adaptive clipping method, and GPG~\citep{chu2025gpg} as a clipping-free baseline. 
For the latter, we compare to the vanilla variant (GPG) and also add the \gls{troll} projection to its policy update (GPG (\textit{TROLL})).}

\textbf{Experiment Setup.}
\label{ssec:experiment_setup}
We base our experiments on the \texttt{verl} repository\footnote{\url{https://github.com/volcengine/verl}}, using default parameters and training recipes where applicable.
We set the group size for the advantage normalization of all methods to $8$. 
We use a token-level loss aggregation~\citep{yu2025dapoopensourcellmreinforcement} for \gls{ppo} and \gls{grpo}, and opt for method-specific loss aggregations for \gls{drgrpo} and \gls{gspo}.
We evaluate the test datasets every $10$ steps.
To improve visibility, we use sliding windows of size $7$ and $21$ for the train and test evaluations, respectively, while showing the unsmoothed values in the background.
\autoref{app_sec:setup} provides additional details, including hyperparameters in \autoref{tab:hp}. 
\autoref{app_sec:additional_results} shows all results. 

\begin{table}[t]
    \centering
\resizebox{\textwidth}{!}{%
\begin{tabular}{llccccc|ccccc}
& & \multicolumn{5}{c|}{Qwen3-$8$B} & \multicolumn{5}{c}{Qwen2.5-$7$B-Instruct} \\ 
& & GRPO & Dr.GRPO & PPO & GSPO & \revision{RF++} & GRPO & Dr.GRPO & PPO & GSPO & \revision{RF++}   \\
\midrule
\texttt{DAPO} & \textit{Clip}  
    & $0.667$ & $0.678$ & $0.640$ & $0.000$ & \revision{$0.648$} & $0.443$ & $0.467$ & \cellcolor{blue!15}$0.444$ & $0.159$ & \revision{$0.429$}  \\
Train & \textit{TROLL} 
    & \cellcolor{blue!15}$0.721$ & \cellcolor{blue!15}$0.704$ & \cellcolor{blue!15}$0.744$ & \cellcolor{blue!15}$0.736$ & \cellcolor{blue!15}\revision{$0.742$}
    & \cellcolor{blue!15}$0.495$ & \cellcolor{blue!15}$0.513$ & $0.431$ & \cellcolor{blue!15}$0.481$  & \cellcolor{blue!15}\revision{$0.486$}\\
\midrule
\texttt{DAPO} & \textit{Clip}  
    & $0.640$ & $0.653$ & $0.602$ & $0.000$ & \revision{$0.626$} & $0.323$ & $0.331$ & $0.324$ & $0.093$  & \revision{$0.323$} \\
Eval. & \textit{TROLL} 
    & \cellcolor{blue!15}$0.691$ & \cellcolor{blue!15}$0.674$ & \cellcolor{blue!15}$0.715$ & \cellcolor{blue!15}$0.706$ &\cellcolor{blue!15}\revision{$0.728$}
    & \cellcolor{blue!15}$0.389$ & \cellcolor{blue!15}$0.389$ & \cellcolor{blue!15}$0.353$ & \cellcolor{blue!15}$0.390$ & \cellcolor{blue!15}\revision{$0.380$} \\
\midrule
\texttt{MATH} & \textit{Clip}  
    & $0.541$ & \cellcolor{blue!15}$0.549$ & $0.508$ & $0.000$ & \revision{$0.520$} & $0.313$ & $0.317$ & $0.319$ & $0.127$ & \revision{$0.311$} \\
Eval. & \textit{TROLL} 
    & \cellcolor{blue!15}$0.551$ & $0.546$ & \cellcolor{blue!15}$0.591$ & \cellcolor{blue!15}$0.580$ & \cellcolor{blue!15}\revision{$0.578$}
    & \cellcolor{blue!15}$0.350$ & \cellcolor{blue!15}$0.359$ & \cellcolor{blue!15}$0.349$ & \cellcolor{blue!15}$0.333$ &  \cellcolor{blue!15}\revision{$0.344$} \\
\bottomrule
\end{tabular}
}
    \caption{
    Final train and evaluation success rates on \texttt{DAPO} for Qwen3-$8$B and Qwen2.5-$7$B-Instruct methods for different advantage estimation methods for \textit{\gls{troll}} and \textit{Clip}.
    The better approach is marked in \cellcolor{blue!15}blue.
    \textit{\gls{troll}} significantly improves over \textit{Clip} in most cases, and is able to successfully train \gls{gspo}, where \textit{Clip} causes divergence and little to no success rates on both models.
    }
    \label{tab:qwen_methods}
\end{table}

\section{Results}
\label{sec:results}

\textbf{Qwen Experiments on} \texttt{DAPO-Math}.
\label{ssec:results_qwen}
We evaluate models from the Qwen3 and Qwen2.5-Instruct families~\citep{qwen2025qwen25technicalreport, yang2025qwen3technicalreport} on \texttt{DAPO} \citep{yu2025dapoopensourcellmreinforcement}. 
\revision{The top of} \autoref{fig:qwen3_size} compares \textit{TROLL} and the \textit{Clip} objective for different Qwen3 model sizes optimized with GRPO~\citep{shao2024deepseekmath}.
\textit{TROLL} consistently improves training performance, causing more sample-efficient training and improved success rates at convergence for all models.
These results directly translate to both evaluation sets, \texttt{DAPO}-Eval and \texttt{MATH}-Eval. 
Interestingly, the $4$B \textit{TROLL} model almost matches the performance of the $14$B \textit{Clip} one.
\autoref{app_fig:qwen2_5_size} in \autoref{app_ssec:qwen_results} shows similar performance trends across Qwen2.5-Instruct model sizes. 
The right of \autoref{fig:figure1} further compares the runtime of both variants on Qwen3-$14$B, showing that \gls{troll}'s projections do not incur a significant computational overhead.
Finally, \autoref{app_sec:example_generations} provides example sequences generated by Qwen3-$14$B on a \texttt{MATH}-Eval problem.

\autoref{tab:qwen_methods} compares \textit{\gls{troll}} and \textit{Clip} results for Qwen3-$8$B and Qwen2.5-$7$B-Instruct for \gls{grpo}, \gls{drgrpo}, \gls{ppo}, \gls{gspo}\revision{, and \gls{rfpp}}.
\textit{\gls{troll}} generally improves success rates by $3$-$10$ \revision{percentage points} across methods and evaluated datasets.
\revision{In these experiments, choosing \textit{\gls{troll}} over \textit{Clip} is usually more beneficial than selecting any of the considered advantage estimation methods.}
\autoref{app_tab:qwen_methods_tasks} provides results for the individual \texttt{MATH} datasets, while \autoref{app_fig:qwen3_methods} and \autoref{app_fig:qwen2_5_methods} show full training curves for Qwen3-$8$B and Qwen2.5-$7$B-Instruct, respectively.
\revision{Both figures show that \gls{gspo} (\textit{Clip}) eventually diverges during training}, while \gls{gspo} (\textit{\gls{troll}}) remains stable and achieves similar success rates to the other advantage estimation methods.

\revision{\textbf{Qwen Experiments on} \texttt{Eurus-Code}.
The bottom of \autoref{fig:qwen3_size} shows that \textit{\gls{troll}} is directly applicable to code generation tasks, yielding substantial advantages over \textit{Clip} for all model sizes. 
For instance, our Qwen3-$1.7$B (\textit{\gls{troll}}) improves over Qwen3-$4$B (\textit{Clip}), and significantly outperforms the model of its own size.
Specifically, we see improvements of $7{-}18$ percentage points, which translate to a $18{-}30\%$ relative gain, when using \textit{\gls{troll}} instead of \textit{Clip}.} 

\revision{\textbf{Additional Model Families.}}
\autoref{fig:other_models} shows various models of different families and sizes on \texttt{GSM8K}, again indicating a clear benefit for \textit{\gls{troll}} over the \textit{Clip} objective.
Here, models of the Llama family often need a significant number of training steps before \textit{Clip} shows a positive training signal, while \textit{\gls{troll}} causes the models to start learning much faster.
\revision{\textit{\gls{troll}} further enables stable and fast training for Apertus, while \textit{Clip} fails in some setups.}
\autoref{app_ssec:more_models} provides additional results on more models and the \texttt{GSM8K} dataset.
We omit evaluations for \texttt{DAPO} with models other than SmolLM3-$3$B, as none matched the performance of Qwen3-$1.7$B in preliminary \textit{Clip} experiments.

\revision{\textbf{Additional Math Datasets.}}
Considering other \revision{math} datasets, the top rows on the left of \autoref{fig:other_models} shows that \textit{\gls{troll}} is also beneficial on other datasets, as evaluated on \revision{\texttt{Eurus-Math}} for Qwen3-$8$B and the simpler \texttt{GSM8K} for Qwen3-$0.6$B.
\autoref{app_ssec:dataset_results} provides detailed success rates for \revision{\texttt{Eurus-Math}} in \autoref{app_fig:eurus_qwen3_8b} and additional results on \texttt{GSM8K} for larger Qwen3 models in \autoref{app_fig:gsm8k_qwen3}. 

\revision{\textbf{Additional Training Algorithms.}
\autoref{app_fig:qwen3_baselines} in \autoref{app_ssec:additional_training_algos} shows that adaptive clipping via BAPO~\citep{xi2025bapo} slightly improves over regular \textit{Clip} on evaluation data, but still clearly underperforms \textit{\gls{troll}}.
It also shows preliminary results on how \textit{\gls{troll}} can benefit non-clipping-based policy gradient methods such as GPG~\citep{chu2025gpg}. 
While vanilla GPG suffers from stability issues in our experiments, these issues are resolved by adding our differentiable trust region projections, resulting in performance comparable to GRPO (\textit{\gls{troll})}.}

\begin{figure}
\definecolor{darkgray176}{RGB}{176,176,176}
\definecolor{lightgray204}{RGB}{204,204,204}
\definecolor{steelblue49130189}{RGB}{49,130,189}
\definecolor{crimson2224538}{RGB}{222,45,38}
\definecolor{darkgray176}{RGB}{176,176,176}
\definecolor{orange}{RGB}{255,165,0}

\begin{tikzpicture}
    \begin{axis}[
        hide axis,
        xmin=0,
        xmax=1,
        ymin=0,
        ymax=1,
        legend columns=8,
        legend cell align=left,
        font=\small,
        legend style={
            draw=none,
            column sep=1ex,
            line width=1pt
        }
    ]
    
    \addlegendimage{ultra thick, orange}
    \addlegendentry{Apertus-$8$B Instr.}
    \addlegendimage{ultra thick, crimson2224538} 
    \addlegendentry{Llama3.2-$3$B}
    \addlegendimage{ultra thick, steelblue49130189}
    \addlegendentry{SmolLM-$3$B}
    
    \addlegendimage{empty legend}
    \addlegendentry{~}
    \addlegendimage{ultra thick, gray, dashed}
    \addlegendentry{\textit{Clip}}
    \addlegendimage{ultra thick, gray}
    \addlegendentry{\textit{TROLL}}
    
    \end{axis}
\end{tikzpicture}\\
    \begin{minipage}{0.53\textwidth}
    \vspace{-\baselineskip}
    \small{
    \begin{tabularx}{\linewidth}{l>{\centering\arraybackslash}Xcc}
    \toprule
    Model & \texttt{Dataset} & \textit{Clip} & \textit{TROLL} \\
    \midrule
    Qwen3-$0.6$B & \texttt{GSM8K} & 0.828 & \cellcolor{blue!15}0.833 \\
    Qwen3-$8$B & \mbox{\hspace{-0.9em} \texttt{\revision{Eurus-Math}}} & 0.561 & \cellcolor{blue!15}0.579  \\
    \midrule
    SmolLM3-$3$B & \texttt{GSM8K} & 0.915 & \cellcolor{blue!15}0.925 \\
    SmolLM3-$3$B & \texttt{DAPO} & 0.580 & \cellcolor{blue!15}0.606 \\
    \midrule 
    Llama3.2-$3$B   & \texttt{GSM8K} & 0.589 & \cellcolor{blue!15}0.668 \\
    Llama3.2-$3$B Instr.  & \texttt{GSM8K} & 0.836 & \cellcolor{blue!15}0.850 \\ 
    FineMath-Llama-$3$B & \texttt{GSM8K} & \cellcolor{blue!15}0.750 & 0.746 \\ 
    \midrule
    Llama3.1-$8$B  & \texttt{GSM8K} & 0.000 & \cellcolor{blue!15}0.759 \\ 
    Llama3.1-$8$B Instr.  & \texttt{GSM8K} & 0.855 & \cellcolor{blue!15}0.872 \\    
    \midrule
    Apertus-$8$B & \texttt{GSM8K} & 0.156 & \cellcolor{blue!15}0.697 \\
    Apertus-$8$B Instr. & \texttt{GSM8K} & 0.688 & \cellcolor{blue!15}0.824 \\
    \bottomrule
    \end{tabularx}
    }
    \end{minipage}%
    \begin{minipage}{0.47\textwidth}  
    \input{fig/other_models_main/gsm8k_cherry}
    \begin{tikzpicture}

\definecolor{darkgray176}{RGB}{176,176,176}
\definecolor{lightgray204}{RGB}{204,204,204}
\definecolor{steelblue49130189}{RGB}{49,130,189}

\begin{axis}[
width=\textwidth,
height=3.2cm,
legend cell align={left},
legend style={
  font=\small,
  fill opacity=0.8,
  draw opacity=1,
  text opacity=1,
  at={(0.97,0.03)},
  anchor=south east,
  draw=lightgray204
},
tick align=outside,
tick pos=left,
title={\texttt{DAPO} Eval},
x grid style={darkgray176},
xlabel={Training Step},
xmajorgrids,
ylabel={Success Rate},
xmin=-4.5, xmax=254.5,
xtick style={color=black},
y grid style={darkgray176},
xlabel style={font=\small},
ylabel style={font=\small},
ymajorgrids,
ymin=0.441046966731898, ymax=0.640166340508806,
tick label style={font=\small},
yticklabel style={
/pgf/number format/fixed,
/pgf/number format/fixed zerofill,
/pgf/number format/precision=2
},
ylabel shift = -0.1cm,
xlabel shift = -0.1cm,
title style={yshift=-0.2cm, xshift=-0.2cm, font=\small}]
]
]
\addplot [ultra thick, steelblue49130189, opacity=0.2, dashed, forget plot]
table {%
10 0.450097847358121
20 0.48238747553816
30 0.515655577299413
40 0.524461839530333
50 0.547945205479452
60 0.562622309197652
70 0.563600782778865
80 0.587084148727984
90 0.581213307240705
100 0.593933463796477
110 0.583170254403131
120 0.585127201565558
130 0.589041095890411
140 0.578277886497065
150 0.580234833659491
160 0.584148727984344
170 0.580234833659491
180 0.585127201565558
190 0.582191780821918
200 0.574363992172211
210 0.598825831702544
220 0.576320939334638
230 0.575342465753425
240 0.585127201565558
250 0.577299412915851
260 0.580234833659491
270 0.577299412915851
280 0.586105675146771
290 0.577299412915851
300 0.580234833659491
};
\addplot [ultra thick, steelblue49130189, opacity=0.2, forget plot]
table {%
10 0.484344422700587
20 0.545988258317025
30 0.551859099804305
40 0.600782778864971
50 0.588062622309198
60 0.610567514677104
70 0.607632093933464
80 0.611545988258317
90 0.602739726027397
100 0.622309197651663
110 0.603718199608611
120 0.60958904109589
130 0.61252446183953
140 0.620352250489237
150 0.616438356164384
160 0.613502935420744
170 0.620352250489237
180 0.631115459882583
190 0.596868884540117
200 0.602739726027397
210 0.62426614481409
220 0.590998043052838
230 0.602739726027397
240 0.602739726027397
250 0.597847358121331
260 0.604696673189824
270 0.605675146771037
280 0.617416829745597
};
\addplot [ultra thick, steelblue49130189]
table {%
10 0.545743639921722
20 0.554207436399217
30 0.563600782778865
40 0.569890970086665
50 0.588062622309198
60 0.596169974839251
70 0.606234274531731
80 0.60665362035225
90 0.609728823036064
100 0.61000838691641
110 0.611825552138664
120 0.61252446183953
130 0.614062063181437
140 0.61378249930109
150 0.617696393625944
160 0.61587922840369
170 0.614481409001957
180 0.61504053676265
190 0.611406206318144
200 0.609868604976237
210 0.607352530053117
220 0.602599944087224
230 0.603718199608611
240 0.604137545429131
250 0.603159071847917
260 0.605185909980431
270 0.605675146771037
280 0.606409001956947
};
\addplot [ultra thick, steelblue49130189, dashed]
table {%
10 0.493150684931507
20 0.504109589041096
30 0.513861709067189
40 0.520967291025999
50 0.540536762650266
60 0.554654738607772
70 0.565837293821638
80 0.574224210232038
90 0.579535923958625
100 0.583310036343304
110 0.585406765445904
120 0.584428291864691
130 0.584847637685211
140 0.582890690522785
150 0.583170254403131
160 0.582750908582611
170 0.580654179480011
180 0.583589600223651
190 0.583030472462958
200 0.581772435001398
210 0.582471344702265
220 0.581353089180878
230 0.581073525300531
240 0.581492871121051
250 0.579675705898798
260 0.579815487838971
270 0.580514397539838
280 0.579745596868885
290 0.580234833659491
300 0.580234833659491
};
\end{axis}

\end{tikzpicture}
    \end{minipage}
    \vspace{-0.7cm}
    \caption{
    (\textbf{Left}) Final evaluations for \textit{\gls{troll}} and \textit{Clip} for different combinations of models and datasets trained with \gls{grpo}.
    The better approach between \textit{\gls{troll}} and \textit{Clip} is marked in \cellcolor{blue!15}blue.
    (\textbf{Right}) Comparison of \textit{TROLL} (full lines) and the \textit{Clip} objective (dashed lines) for different models trained with GRPO.
    \textit{\gls{troll}} generally improves over \textit{Clip}, and performs well across all considered datasets. 
    In particular, \textit{\gls{troll}} leads to significantly faster learning for different Llama models, where \textit{Clip} often takes significantly more iterations to obtain a positive training signal.
    \textit{\gls{troll}} also showcases more stable performance compared to \textit{Clip} throughout training.
    }
    \label{fig:other_models}
    \vspace{-0.3cm}
\end{figure}

\section[Analysis and Parameter Study]{Analysis \revision{and Ablations}}
\label{ssec:analysis}

\textbf{KL Bounds and Sparsity Threshold \revision{Experiments}.}
We explore different values for the \gls{kl} bound $\epsilon$ and the maximum number of kept tokens $K$ in the sparsification process for Qwen3-$8$B trained with \gls{grpo} on the \texttt{DAPO} dataset.
The left of \autoref{fig:ablations_analysis} finds that a lower \gls{kl} slows down training but does not affect convergence, while a higher \gls{kl} leads to worse success rates, likely due to too large policy updates.
A small $K{=}16$ causes poor updates, presumably due to poor estimates of the underlying dense distributions, while a larger $K{=}256$ increases cost but does not improve over our default $K{=}64$.
\autoref{app_fig:ablations} in \autoref{app_ssec:ablations} provides additional detail.
These results suggest that \textit{\gls{troll}} requires an accurate \gls{kl} projection, while showing that there is a wide range of suitable hyperparameters for both the \gls{kl} bound and the sparsification.
Finally the top row of \autoref{app_fig:analysis} shows that $5{-}10$ tokens are usually sufficient to capture $99.999\%$ of logit probabilities.

\revision{\textbf{Batch Size.}
\autoref{app_fig:batch_size} in \autoref{app_ssec:batch_size} compares the training behavior of \textit{\gls{troll}} and \textit{Clip} for different batch sizes. 
We find that \textit{\gls{troll}} can easily deal with larger batch sizes and thus less recent data in its optimization, while \textit{Clip} clearly degrades when increasing the batch size.}

\textbf{Output Diversity and Entropy.} 
Recent work has shown that the \gls{ppo}-like \textit{Clip} objective tends to exploit the \gls{llm}'s existing knowledge by reducing each token distribution's entropy to increase the reward~\citep{cui2025entropy}.
In contrast, the bottom right of \autoref{fig:ablations_analysis} shows that \gls{troll} preserves entropy.
\revision{Additionally, \autoref{app_fig:eurus_code_ent} in \autoref{app_ssec:analysis} shows a clear correlation between \gls{troll}'s ability to preserve entropy and improve performance on \texttt{Eurus-Code}.}

\textbf{Projection Fraction.}
Comparing the fraction of clipped tokens for \textit{Clip} with the fraction of projected tokens for {\gls{troll} shows that both trust region approaches roughly affect the same number of tokens.
The observed stability improvements are thus not merely caused by more restrained tokens.
We compare both ratios for larger Qwen3 models in the middle row of \autoref{app_fig:analysis}.

\textbf{Response Length.} 
The lower row of \autoref{app_fig:analysis} shows that \gls{troll} adapts response length more quickly to ranges suitable for solving the tasks. 
This faster adaption reflects the faster performance improvements achieved by \gls{troll}.

\textbf{Computational Overhead.} \autoref{app_ssec:troll_overhead} provides a controlled experiment setup for measuring \gls{troll}'s computational overhead. 
We find on the top right table of \autoref{fig:ablations_analysis} that the memory overhead of maintaining sparse distributions is negligible compared to storing and backpropagating through the \gls{llm}, as explained in \autoref{app_eq:naive_sparse_overhead}. 
Further, both memory and computation time for \gls{troll} scale only with the vocabulary size, which is constant for most model families. We thus find that \gls{troll}'s overhead diminishes as model size increases.
\autoref{app_tab:troll_overhead} provides detailed evaluations.

\begin{figure}
    \centering
    \begin{minipage}[c]{0.5\textwidth}
    \begin{tikzpicture}
\definecolor{darkcyan0119182}{RGB}{0,119,182}
\definecolor{darkgray176}{RGB}{176,176,176}
\definecolor{darkgreen279432}{RGB}{27,94,32}
\definecolor{darkorange25511126}{RGB}{255,111,26}
\definecolor{lightgray204}{RGB}{204,204,204}
\definecolor{saddlebrown168670}{RGB}{168,67,0}
\definecolor{yellowgreen12621787}{RGB}{126,217,87}
\begin{axis}[
        hide axis,
        xmin=0,
        xmax=1,
        ymin=0,
        ymax=1,
        legend columns=3,
        legend cell align=left,
        font=\small,
        legend style={
            draw=none,
            column sep=1ex,
            line width=1pt
        }
    ]
\addlegendimage{ultra thick, darkcyan0119182}
\addlegendentry{\textit{TROLL}}
\addlegendimage{ultra thick, yellowgreen12621787}
\addlegendentry{$\epsilon=0.01$}
\addlegendimage{ultra thick, darkgreen279432}
\addlegendentry{$\epsilon=0.25$}

\addlegendimage{empty legend}
\addlegendentry{~}
\addlegendimage{ultra thick, darkorange25511126}
\addlegendentry{$K=256$}
\addlegendimage{ultra thick, saddlebrown168670}
\addlegendentry{$K=16$}
\end{axis}
\end{tikzpicture}
    \input{fig/analysis_main/agg_MATH_Eval_runtime_no_legend}
    \end{minipage}%
    \begin{minipage}[c]{0.5\textwidth}
    \centering
        \small{
        \begin{tabular}{lcc}
         Qwen3-$4$B        &  \textit{Clip} & \textit{Troll}  \\
                 \toprule
         VRAM    & $34.574$~GiB & $36.157$~GiB \\
         Runtime & $85.133$~s  & $92.906$~s   \\ 
         \bottomrule        
        \end{tabular}
    }
    \vfill
    \input{fig/analysis_main/entropy}
    \end{minipage}
    \caption{
        (\textbf{Left}) Qwen3-$1.7$B trained with GRPO using the \textit{\gls{troll}} projection compared to different hyperparameter choices.
        \textit{\gls{troll}} works well for conservative \gls{kl} bounds $\epsilon$ and top-$K$ logit selections, but is slower for too conservative values and degrades slightly for too aggressive updates or token pruning.
        (\textbf{Top Right}) Memory and runtime comparison between \textit{\gls{troll}} and \textit{Clip} in a controlled environment. \textit{\gls{troll}} imposes a modest overhead compared to the cost of training the \gls{llm} parameters.
        (\textbf{Bottom Right}) \textit{\gls{troll}} generally maintains more entropy during training while showing higher success rates when compared to \textit{Clip}, as shown for Qwen3-$14$B.
    }
    \label{fig:ablations_analysis}
\end{figure}

\section{Conclusion}
\label{sec:conclusion}

We introduce \glsfirst{troll}, a trust-region based policy gradient objective that acts as a drop-in replacement for the popular \gls{ppo}-clip.
\gls{troll} is based on a novel, principled, and fully differentiable trust-region projection for discrete distributions.
This projection compares two distributions, in our case, the token logit distributions of an old policy used to collect sequences, and a new policy that performs policy gradient updates on these sequences.
Since these distributions are prohibitively large for modern vocabulary sizes, we extend the projection to sparse distributions. 
Here, we only keep a small subset of logits that represent the most likely token predictions, allowing us to realize both data collection and the projection objective using fully sparse operations.
We experimentally validate \gls{troll} across various model families, model sizes, advantage estimation methods, and \revision{math and code generation} datasets. 
\gls{troll} \revision{significantly and} consistently outperforms the \gls{ppo}-clip objective in terms of sample efficiency and final reward across setups, while only requiring a small overhead that does not scale with model size. 

\textbf{Limitations and Future Work.}
We currently evaluate our method on dense models with up to $14$B parameters. In future work, we want to scale \gls{troll} to larger models and Mixture-of-Experts architectures.
Similarly, it would be interesting to extend \gls{troll} to other modalities and tasks, using, for example, vision-language models, where the logit distributions and their projections may behave differently from pure language.

\clearpage
\bibliography{iclr2026_conference}
\bibliographystyle{iclr2026_conference}

\appendix
\section*{Ethics Statement}
\gls{troll} improves the efficiency of LLM finetuning by enabling scalable trust-region optimization. 
While our experiments focus on mathematical reasoning, the method is broadly applicable to other domains.
As with any advance in LLM training, this carries both potential benefits and risks, depending on the context of deployment.
We believe that managing and shaping the societal impacts of increasingly powerful LLMs should not be left to individual researchers, organizations, or companies alone, but they must be carefully governed and regulated by sovereign governments and strong democratic institutions.

\section*{Reproducibility Statement}
All experiments in this paper rely on publicly available pretrained checkpoints.
We exclusively use publicly available datasets.
While some were modified, we describe these modifications and will release the processed versions upon the deanonymization of the paper.
Further information, together with additional hyperparameters and training details, are provided in \autoref{app_sec:setup}.
Our implementation builds on open-source repositories and will be made available after deanonymization.

\section*{On LLM Usage}
We used LLMs to assist with revising grammar, style, and text flow in this manuscript.
In addition, we employed LLMs to support aspects of the implementation and generate visualizations for this manuscript. 

\section*{Acknowledgements}
This work was supported by the European Research Council (ERC) under the European Union’s Horizon Europe programme through the project SMARTI³ (Grant Agreement No. 101171393).
The authors acknowledge support by the state of Baden-Württemberg through bwHPC, as well as the HoreKa supercomputer funded by the Ministry of Science, Research and the Arts Baden-Württemberg and by the German Federal Ministry of Education and Research.

\section{Derivations}
\label{app_sec:derivations}
For each output token $o_t$ the trust region projection layer solves 
\begin{align}
   & \underset{\pol}{\arg\min} \KL{\pol}{\tarpol} \label{eq:opt}   \\ 
    \text{s.t.} ~~  &\KL{\pol}{\oldpol} < \epsilon ~~ \text{and} ~~ \sum_{o_t} \left[ \pol \right] = 1. \nonumber
\end{align}
Here, the first constraint enforces the trust region to the previous distribution $\oldpol$ and the second constraint ensures the resulting distribution is properly normalized.
We solve the constrained optimization problem using the method of Lagrangian multipliers and start with the primal solution.

\subsection{Primal Solution}
\label{app_sec:troll_primal}

To compute the primal solution of this optimization problem, we first set up the Lagrangian function by introducing Lagrangian multipliers $\eta > 0$ and $\lambda$ for the first and second constraint, respectively.
The corresponding Lagrangian is given as
\begin{align}
\mathcal{L}&(\pol, \eta)  \nonumber \\ =& \KL{\pol}{\tarpol} - \eta (\epsilon - \KL{\pol}{\oldpol} \nonumber \\  & - \lambda \left(1 -  \sum_{o_t} \left[\pol  \right]  \right) \nonumber \\
=& - (\eta \epsilon + \lambda) + \sum_{o_t} \left[\pol  \left( \log \dfrac{\pol}{\tarpol} + \eta \log \dfrac{\pol}{\oldpol} + \lambda \right) \right] \nonumber \\
=& -(\eta \epsilon + \lambda) +  \label{eq:lagrangian}\\ 
& \!\sum_{o_t}  \bigl[ \pol \bigl( (\eta + 1) \log \pol - \log \tarpol - \eta \log \oldpol + \lambda \bigr) \bigr] \nonumber.
\end{align}
We can now obtain the optimal primal solution to \autoref{eq:opt} by taking the derivative of the Lagrangian w.r.t. $\pol$, setting it to $0$, and solving for $\pol$. 
The derivative is given by
\begin{align*}
    & \dfrac{\partial \mathcal{L}(\pol, \eta)}{ \partial \pol}  \\ 
    =& \sum_o \left[ (\eta + 1) + (\eta + 1) \log \pol - (\log \tarpol + \eta \log \oldpol) + \lambda \right].
\end{align*}

Clearly $\nicefrac{\partial \mathcal{L} (\pol, \eta)}{\partial \pol} = 0 $ if all the individual terms of the sum are $0$. Thus, the problem simplifies to 
\begin{align*}
  0 = (\eta + 1) + (\eta + 1) \log \pol - (\log \tarpol + \eta \log \oldpol) + \lambda
\end{align*}
which yields 
\begin{align}
\log \pol = \dfrac{\log \tarpol + \eta \log \oldpol) - (\eta + 1 + \lambda)}{\eta + 1} \label{eq:full_log_primal_sol}
\end{align}
and thus 
\begin{align}
\pol & = \exp \left( \dfrac{\log \tarpol + \eta \log \oldpol}{\eta + 1} \right) \exp\left( - \dfrac{\eta + 1 + \lambda}{\eta + 1}\right) \nonumber \\ 
& \propto \exp \left( \dfrac{\log \tarpol + \eta \log \oldpol}{\eta + 1} \right) \label{eq:primal_sol}
\end{align}
Crucially, this primal solution allows computing a properly normalized distribution $\pol$ without explicitly computing $\lambda$ by replacing the $\exp$ in \autoref{eq:primal_sol} with a softmax. 

\subsection{Dual Solution}
\label{app_sec:troll_dual}
The second step of the Lagrangian multiplier method is to solve the dual problem which finds the optimal dual parameters given the primal solution.
To that end, we insert the primal solution from \autoref{eq:full_log_primal_sol} into the Lagrangian (\autoref{eq:lagrangian}). Here most terms cancel out, leading to a dual of the form 
\begin{align}
    D(\eta, \lambda) = - \eta\epsilon -\lambda - \eta - 1 = -\eta\epsilon - (\eta + 1 + \lambda) \label{eq:dual_part1}.
\end{align}
In a second step towards a practically usable dual, we remove the dependency on $\lambda$ by exploiting the constraint it enforces, i.e., $\sum_{o_t}\left[\pol\right] = 1$. Going to log space and again using \autoref{eq:full_log_primal_sol}, this property yields
\begin{align*}
    0 & = \log \sum_{o_t}\left[\pol\right] \\ & =  \log \sum_{o_t} \left[ \exp \left( \dfrac{\log \tarpol + \eta \log \oldpol}{\eta + 1} \right) \exp\left( - \dfrac{\eta + 1 + \lambda}{\eta + 1}\right) \right]  \\ 
    & = - \dfrac{\eta + 1 + \lambda}{\eta + 1} + \log \sum_{o_t} \left[ \exp \left( \dfrac{\log \tarpol + \eta \log \oldpol}{\eta + 1} \right) \right]
\end{align*}
which we can rewrite as
\begin{align}
 \eta + 1 + \lambda  = (\eta + 1)\log \sum_{o_t} \left[ \exp \left( \dfrac{\log \tarpol + \eta \log \oldpol}{\eta + 1} \right) \right] \label{eq:dual_part2}.
\end{align}
Now, inserting \autoref{eq:dual_part2} into \autoref{eq:dual_part1} removes the dependency on $\lambda$ leading to
\begin{align*}
    D(\eta) = - \eta \epsilon - (\eta +1)\log \sum_{o_t} \left[ \exp \left( \dfrac{\log \tarpol + \eta \log \oldpol}{\eta + 1} \right) \right].
\end{align*}
Using this dual, we can find the optimal $\eta^*$ by solving
\begin{align}
    \eta^* = \arg \max_\eta D(\eta) ~~ s.t. ~~ \eta \geq 0 \label{eq:final_dual}. 
\end{align}
We can efficiently optimize this scalar optimization problem using the $n$-ary bracketing method described in \autoref{lst:brack}. 

\subsection{Gradients}\label{app_subsec:proj_gradient}

This trust region projection is trivially differentiable using standard autograd tools, except for the numerical optimization of the dual to find the optimal step size $\eta^*$. 
Towards differentiating through this optimization problem in closed form, let us first change perspective and no longer consider the distributions $\pol$, $\oldpol$, and $\tarpol$ directly but vectors $\probs$, $\oldlogits$, and $\tarlogits$.
Here $q$ corresponds to the probabilities of $\pol$ while $\oldlogits$ and $\tarlogits$ denote to the normalized logits of $\oldpol$ and $\tarpol$. 
We further assume all $3$ vectors are normalized, i.e.,
$$
\sum \probs = 1, \quad \sum \exp{\oldlogits} =1 \quad \textrm{and} \quad \sum \exp{\tarlogits} = 1.
$$
While this notation may seem slightly unintuitive at first, it simplifies the following derivations.
As we assume the $\oldpol$ and consequently $\oldlogits$ are constant, the only gradient we are interested in is $\dfrac{\partial \eta^*}{\partial \tarlogits}$, i.e., how the original LLM's output influences the optimal step size $\eta^*$. 
Since we do not have an analytical form for the optimal step size $\eta^*$ but only the result of the numerical optimization, we need to introduce further analytical properties.
Using the implicit differentiation~\citep{dontchev2009implicit} and differentiable matrix calculus~\citep{magnus1989matrix} techniques introduced to deep learning by OptNet~\citep{amos2017optnet}, we start by writing out the Karush–Kuhn–Tucker (KKT) conditions~\citep{karush1939minima} of the dual in \autoref{eq:final_dual} for the optimum at $\eta^*$. 
Denoting the Lagrangian multiplier corresponding to the $\eta \geq 0$ constraint by $\mu$ and realizing that $\nabla D(\eta) = \epsilon - \KL{\pol}{\oldpol} =\eps - \probs^T( \log \probs - \oldlogits)$, those are given by 
\begin{align*}
    \underbrace{\nabla D(\eta^*) + \mu \nabla(-\eta^*) = \epsilon -  \probs^T( \log \probs - \oldlogits)- \mu =  0}_{\textrm{Stationarity}} \qquad \text{and} \quad \underbrace{ \mu (-\eta^*) = 0}_{\textrm{Complementary Slackness}}\!\!.
\end{align*}
As there is no equality constraint in \autoref{eq:final_dual}, primal feasibility is given by default. 
We can now take the total differentials around these conditions, which are given by
\begin{align}
    0 &= d \left( \epsilon -  \probs^T( \log \probs - \oldlogits) - \mu \right) = - d \left( \probs^T( \log \probs - \oldlogits) \right) - d\mu \label{stat:1}\\ 
    0 &= d(\mu(-\eta^*) = d\mu(-\eta^*) + \mu(-d\eta^*)    \label{cs:1},
\end{align}
where $d\epsilon$ vanishes as it is constant. 
Before proceeding, we need to rewrite the KL term $d \left( \probs^T(\log \probs - \oldlogits) \right)$ in terms of $\tarlogits$ and simplify.
First, we have 
\begin{align}
    d \left( \probs^T(\log \probs - \oldlogits) \right) = (1 + \log \probs - \oldlogits)^T d\probs \label{eq:diff1}
\end{align}
and need to continue with the differential $d\probs$.
Again using the primal solution \autoref{eq:primal_sol}, we get
\begin{align}
\probs = \textrm{softmax} \left( \dfrac{\eta^* \oldlogits + \tarlogits}{\eta^* + 1} \right). \label{eq:primal_as_softmax}
\end{align}
Assuming the old logits are a constant, we can write the corresponding differential as 
\begin{align*}
    d \probs = \dfrac{\partial \probs}{\partial \tarlogits} d\tarlogits + \dfrac{\partial q}{\partial \eta^*} d\eta^*.
\end{align*}
Inserting this term into \autoref{eq:diff1} and the plugging the result into \autoref{stat:1} yields
 \begin{align}
 - \left(1 + \log \probs - \oldlogits \right)^T \dfrac{\partial \probs}{\partial \eta^*} d\eta^* - d\mu &= \left( 1 + \log \probs - \oldlogits \right) \dfrac{\partial \probs}{\partial \tarlogits} d \tarlogits  \label{eq:stat2} \\
 -\mu d \eta^* - \eta^* d\mu &= 0,  \label{eq:cs2}
\end{align}
which we can use to compute the desired gradient $\dfrac{\partial \tarlogits}{\partial \eta^*}$. 
To this end, we consider two separate cases. 
First, if the original KL trust region is not violated, then $\eta^* = 0$ and $\mu > 0$. In this case, \autoref{eq:cs2} directly yields that $d\eta^* = 0$ and thus the entire gradient  $ \dfrac{\partial \eta^*}{\partial \tarlogits}$ is zero.
Second, the original KL trust region constraint is active and thus $\eta^* > 0$ and $\mu = 0$. In this case \autoref{eq:cs2} gives $d \mu = 0$ which simplifies \autoref{eq:stat2}. Reordering the remaining terms gives the required gradient
\begin{align*}
    \dfrac{\partial \eta^*}{\partial \tarlogits} = \dfrac{1}{- (1 + \log \probs - \oldlogits)^T \dfrac{\partial \probs}{\partial\eta^*}}  (1 + \log  \probs - \oldlogits)^T \dfrac{\partial \probs}{\partial \tarlogits} 
\end{align*}
The required partial derivatives can be obtained from \autoref{eq:primal_as_softmax} as
\begin{align*}
    \dfrac{ \partial \probs }{\partial \tarlogits}  = \frac{1}{\eta^* + 1} (\textrm{D}(\probs) - \probs\probs^T) \quad \text{and} \quad \dfrac{\partial \probs}{\partial \eta^*} =  \frac{1}{(\eta + 1)^2}  (\textrm{D}(\probs) - \probs\probs^T)(\oldlogits - \tarlogits),
\end{align*}
where $D(q)$ denotes a diagonal matrix with the entries of $\probs$ on the diagonal.

Crucially, for practical purposes, we never need to explicitly materialize the matrices in the partial derivatives.
The resulting backward introduces negligible computational and memory overhead and.
\autoref{lst:dual} shows that, in the non-sparsified case, the backward can be written in less than ten lines of python code.

\subsection{Sparsification}
\label{app_ssec:sparsification_theorems}

\begin{theorem}\label{thm:greater_logit}
    For any pair of logits $o_t^{(1)}$ and $o_t^{(2)}$, with $\tilde{\pi}(o_t^{(1)} \, | \, \bm{q}, \bm{o}_{<t}) \ge \tilde{\pi}(o_t^{(2)} \, | \, \bm{q}, \bm{o}_{<t})$ w.l.o.g., the logit-wise terms that sum to the KL are equally ordered
    \begin{equation}
        \tilde{\pi}(o_t^{(1)} \, | \, \bm{q}, \bm{o}_{<t} ) \log \frac{ \tilde{\pi}(o_t^{(1)} \, | \, \bm{q}, \bm{o}_{<t}) }{ \old{\pi}(o_t^{(1)} \, | \, \bm{q}, \bm{o}_{<t}) } \quad\ge\quad \tilde{\pi}(o_t^{(2)} \, | \, \bm{q}, \bm{o}_{<t}) \log \frac{ \tilde{\pi}(o_t^{(2)} \, | \, \bm{q}, \bm{o}_{<t}) }{ \old{\pi}(o_t^{(2)} \, | \, \bm{q}, \bm{o}_{<t}) }
    \end{equation}
    iff $e^\kappa \ge \gamma$, where $\kappa = \frac{ \tilde{\pi}(o_t^{(1)} \, | \, \bm{q}, \bm{o}_{<t} ) }{ \tilde{\pi}(o_t^{(2)} \, | \, \bm{q}, \bm{o}_{<t} ) }$ is the current probability ratio of the pair and $\gamma$ in $\frac{ \old{\pi}(o_t^{(1)} \, | \, \bm{q}, \bm{o}_{<t}) }{ \old{\pi}(o_t^{(2)} \, | \, \bm{q}, \bm{o}_{<t}) } = \gamma \kappa $ gives the multiplier of the old ratio. 
\end{theorem}
\begin{proof}\label{proof:greater_logit}
    Rewrite \( \tilde{\pi}(o_t^{(1)} \, | \, \bm{q}, \bm{o}_{<t}) \ge \tilde{\pi}(o_t^{(2)} \, | \, \bm{q}, \bm{o}_{<t}) \) as  \(p(x_1) = \kappa \cdot p(x_2) \) for $\kappa \ge 1$ using $p(x_i) = \tilde{\pi}(o_t^{(i)} \, | \, \bm{q}, \bm{o}_{<t})$ for clarity and similarly replace $q(x_i) = \old{\pi}(o_t^{(i)} \, | \, \bm{q}, \bm{o}_{<t})$.
    Then compare the contributions of $x_1$ and $x_2$ to the KL divergence
    \begin{align*}
        \kappa p(x_2) \log \frac{\kappa p(x_2)}{q(x_1)} &\ge p(x_2) \log \frac{p(x_2)}{q(x_2)}\\
        \kappa \log k &\ge \log \frac{q(x_1)}{q(x_2)} 
        \intertext{and substitute \( \frac{q(x_1)}{q(x_2)} \eqcolon \gamma \frac{p(x_1)}{p(x_2)} = \gamma \kappa\)} 
        e^\kappa \kappa &\ge \gamma \frac{p(x_1)}{p(x_2)} \\
        e^\kappa &\ge \gamma.
    \end{align*}
\end{proof}
Here, the assumption that the relative likelihood $\kappa$ of $o_t^{(1)}$ and $o_t^{(2)}$ was not exponentially larger before usually holds in practice, as the token distributions are pushed farther from uniform during training~\citep{cui2025entropy}.

\begin{definition}\label{def:sparsification}
    For any subset $\mathcal{S}$ of the possible tokens, we define $p_{\mathcal{S}}$, or just $p'$ when the mask is clear, as the \emph{sparsed} distribution.
    For tokens not in $\mathcal{S}$, it has default probability $p_d$ and the same probability as $p$ for all others up to the renormalization constant.
    \begin{equation}\label{app_eq:sparse_default_def}
        p_{ \mathcal{S}}(x) = p'(x) \coloneq \begin{cases}
            \gamma p(x),  &\mathrm{for}~x \in \mathcal{S} \\
            p_d ,  &\mathrm{else} 
        \end{cases},
        \qquad
        \gamma = \frac{1 - (|\mathcal{V}| - |\mathcal{S}|) \cdot p_d}{\sum_{x \in \mathcal{S}} p(x)}.
    \end{equation}
    The renormalization factor $\gamma$ accounts for the previous total mass $\sum_{x \in \mathcal{S}} p(x)$ of the kept tokens and new mass $(|\mathcal{V}| - |\mathcal{S}|) \cdot p_d$ of the dropped tokens.
\end{definition}

In the case of equal sparsification masks for distributions $p, q$, we can prove a practically tight upper bound for the true divergence $\KL{p}{q}$ in terms of the sparse divergence $\KL{p'}{q'}$.
\begin{theorem}\label{thm:topk_bound}
    Let $p,q$ be categorical distributions over the vocabulary $|\mathcal{V}|$ with identical top-$k$ logits, $ \operatorname{topk(p)} = \operatorname{topk(q)}$ and equal total probability $\sum_{x \in \operatorname{topk(p)}} p(x) = \sum_{x \in \operatorname{topk(q)}} q(x) = 1-\delta$.
    Then the sparsed distributions $p', q'$ with density 
    \begin{equation*}
        p'(x) \coloneq \begin{cases}
            \gamma p(x),  &\mathrm{for}~x \in \operatorname{topk(p)} \\
            p_d ,  &\mathrm{else}
        \end{cases}, \quad
        q'(x_i) \coloneq \begin{cases}
            \gamma q(x),  &\mathrm{for}~x \in \operatorname{topk(q)} \\
            p_d ,  &\mathrm{else}
        \end{cases},
    \end{equation*}
    and normalization constant $\gamma(\delta, k, |\mathcal{V}|, p_d) \approx 1$ follow the inequality
    \begin{equation*}
        \KL{p}{q} \le \gamma^{-1} \KL{p'}{q'} + \delta \log \frac{ \delta }{ q_{\min} },
    \end{equation*}
    where $q_{\min} = \arg \min_x q(x)$.
\end{theorem}

\begin{proof}
Rename the logits in descending order of probability under p, such that \(p(x_0) \ge p(x_1) \ge \cdots \ge p(x_{|\mathcal{V}|-1})\).
Assume there is $k < |\mathcal{V}|$ such that the largest $k$ logits of both $p$ and $q$ have exactly the total probability mass $\sum_{i=0}^{k} p(x_i) = \sum_{i=0}^{k-1} q(x_i) = 1-\delta$ and the subset of largest logits is identical.
Every nondegenerate distribution has $q_{\min} \le q(x_i)$ and all $p(x_i) \le \delta$ for $i \ge k$, as the total mass could otherwise not be $1-\delta$.
So split the sum over logits in the KL divergence and apply both inequalities
\begin{align*}
    \KL{p}{q} &= \sum_{i=0}^{k-1} p(x_i) \log \frac{ p(x_i) }{ q(x_i) } + \sum_{i=k}^{|\mathcal{V}|-1}  p(x_i) \log \frac{ p(x_i) }{ q(x_i) } \\
    &\le \sum_{i=0}^{k-1} p(x_i) \log \frac{ p(x_i) }{ q(x_i) } + \sum_{i=k}^{|\mathcal{V}|-1}  p(x_i) \log \frac{ \delta }{ q(x_i) } \\
    &\le \sum_{i=0}^{k-1} p(x_i) \log \frac{ p(x_i) }{ q(x_i) } + \sum_{i=k}^{|\mathcal{V}|-1}  p(x_i) \log \frac{ \delta }{ q_{\min} } \\
    &= \sum_{i=0}^{k-1} p(x_i) \log \frac{ p(x_i) }{ q(x_i) } + \log \frac{ \delta }{ q_{\min} } \underbrace{\sum_{i=k}^{|\mathcal{V}|-1}  p(x_i)}_{= \delta} \\
    &= \sum_{i=0}^{k-1} p(x_i) \log \frac{ p(x_i) }{ q(x_i) } + \delta \log \frac{ \delta }{ q_{\min} } .
\end{align*}
Replace p, and analogously q, with their sparsed version as defined in \autoref{def:sparsification},
\begin{equation}
    p'(x_i) \coloneq \begin{cases}
        \gamma p(x_i),  &\mathrm{for}~i < k \\
        p_d ,  &\mathrm{for}~i \ge k 
    \end{cases},
\end{equation}
where \(\gamma = \frac{1 - (|\mathcal{V}| - k) \cdot p_d}{(1 - \delta)} \) renormalizes the $(1-\delta)$ mass of the selected tokens to account for the default mass $(|\mathcal{V}| - k) \cdot p_d$ of the sparsified tokens.
Multiplying with ones and adding a zero to the KL bound yields the relation to the sparse KL, 
\begin{align*}
    \KL{p}{q} &\le \sum_{i=0}^{k-1} \frac{\gamma}{\gamma} p(x_i) \log \frac{ \gamma p(x_i) }{ \gamma q(x_i) } + \delta \log \frac{ \delta }{ q_{\min} } + \gamma^{-1} \sum_{i=k}^{|\mathcal{V}|-1}  p'(x_i) \underbrace{\log \frac{ p_d }{ p_d }}_{=0} \\
    &= \gamma^{-1} \sum_{i=0}^{k-1} p'(x_i) \log \frac{ p'(x_i) }{ q'(x_i) } + \delta \log \frac{ \delta }{ q_{\min} } +  \gamma^{-1} \sum_{i=k}^{|\mathcal{V}|-1}  p'(x_i) \log \frac{ p'(x_i) }{ q'(x_i) } \\
    &= \gamma^{-1} \KL{p'}{q'} + \delta \log \frac{ \delta }{ q_{\min} }.     
\end{align*}
Assuming that $q$'s probabilities can be represented by normal single precision IEEE-754 numbers, $q_{\min} > 1.17549 \cdot  10^{-38}$, and $k \ll |\mathcal{V}|$, e.g. $k=256$ of vocab size $|\mathcal{V}| = 151936$ while using threshold $\delta = 10^{-5}$ and default mass $p_d = 10^{-12}$, the sparse KL approximation,
\begin{align*}
    \KL{p}{q} &\le \frac{(1 - \delta)}{1 - (|\mathcal{V}| - k)\cdot p_d} \KL{p'}{q'} + \delta \log \frac{ \delta }{ q_{\min} } \\
    &= \frac{0.99999}{1 -151680 \cdot 10^{-12}} \KL{p'}{q'} +  10^{-5} \log \frac{ 10^{-5} }{ 1.17549 \cdot  10^{-38} } \\
    &\le 0.99999015168 \cdot \KL{p'}{q'} + 0.00075823623 ,
\end{align*}
is accurate enough for limiting the true divergence to values on the order of $0.05$ as
\begin{align*}
    \KL{p}{q} &\le 0.99999015168 \cdot \KL{p'}{q'} + 0.00075823623 \\
    &\le 0.99999015168 \cdot 0.05 + 0.00075823623 \\
    \KL{p}{q} &\le 0.050757743814.
\end{align*}
\end{proof}

\section{Implementation}\label{app_sec:code}

\begin{figure}[t]
    \centering
    \begin{tikzpicture}[ultra thick]
	
	\draw[rounded corners] (0, 2) rectangle (2.5, 3) node[pos=.5, align=center]{$\tarpol$};
	\draw[rounded corners] (0, 0) rectangle (2.5, 1) node[pos=.5, align=center]{$\oldpol$}; 

    \draw[rounded corners] (4.5, 2) rectangle (5.5, 3) node[pos=.5, align=center]{$\eta^*$}; 

    \draw[rounded corners] (7.5, 2) rectangle (10, 3) node[pos=.5, align=center]{$\pol$};
    
    \draw[rounded corners] (12, 2) rectangle (14, 3) node[pos=.5, align=center]{$\mathcal{J}_\textrm{
    Troll}(\theta)$ \\ \autoref{eq:troll_obj}}; 
    \draw[rounded corners] (12, 0.25) rectangle (14, 1.25) node[pos=.5, align=center]{Ratio Term}; 
    \draw[rounded corners] (12, 3.75) rectangle (14, 4.75) node[pos=.5, align=center]{KL Term};

    \draw[->, >=stealth, red] (2.5, 2.5) -- (4.5, 2.5) node[midway, above, black]{\autoref{eq:final_dual}};
    \draw[->, >=stealth] (5.5, 2.5) -- (7.5, 2.5) node[midway, above, black]{\autoref{eq:primal}};
    \draw[->, >=stealth, dashed] (10, 2.75) -- (11, 2.75) -- (11, 4) -- (12,4); %
    \draw[->, >=stealth] (10, 2.25) -- (11, 2.25) -- (11, 1) -- (12,1);

    \draw[->, >=stealth] (13, 3.75) -- (13,3); 
    \draw[->, >=stealth] (13, 1.25) -- (13,2); 

    \draw[->, >=stealth, rounded corners, dashed] (2.5, 0.5) -- (5, 0.5) -- (5, 2);
    \draw[->, >=stealth, rounded corners, dashed] (2.5, 0.5) -- (12, 0.5);

    \draw[->, >=stealth, rounded corners, dashed] (2.5, 0.5) -- (8.75, 0.5) -- (8.75, 2);
    
    \draw[->, >=stealth, rounded corners] (1.25, 3) -- (1.25, 4.5) -- (8.75, 4.5) -- (8.75, 3);
    \draw[->, >=stealth, rounded corners] (1.25, 3) -- (1.25, 4.5) -- (12, 4.5);
    
    \draw[red] (2.5, -1) -- (4.5, -1);
    \draw[dashed] (10.5, -1) -- (12.5, -1);
    \draw (6.5, -1) -- (8.5, -1);
    \draw[] (3.5, -0.5) -- (3.5, -0.5) node[align=center]{Numerical Optimization};
    \draw[] (7.5, -0.5) -- (7.5, -0.5) node[align=center]{Analytical Computation};
    \draw[] (11.5, -0.5) -- (11.5, -0.5) node[align=center]{No Gradient};
\end{tikzpicture}
    \caption{
    \revision{Compute Graph from the LLM output $\tarpol$ to the RL objective $\mathcal{J}_\text{Troll}(\theta)$. First, the optimal step size $\eta^*$ is computed for each token. Here, we use the fact that the step size for tokens that do not violate the trust region is trivially $0$. For the tokens that do violate the trust region, we need to optimize a $1$-D convex optimization problem to compute $\eta^*$.
    Next, the projection computes the optimal distribution within the trust region, $\pol$, which is then used in the objective.
    This objective combines the standard policy ratio term from \gls{ppo} with a KL term to regress $\tarpol$ towards $\pol$. 
    }
    }
    \label{fig:flow_chart}
\end{figure}

While the theoretical derivation of the differentiable trust region projection can look complex, the final implementation is fairly straightforward.
We give PyTorch-adjacent pseudocode for the dense variant of the primal (\autoref{lst:primal}) and dual (\autoref{lst:dual}) in the following.
Note that the sparse implementation mostly differs in the usage of a custom sparse tensor class that maintains a default probability for the implicit entries.
While this requires additional care in terms of indexing and allows for optimizations of, e.g., the \gls{kl} computation, the general logic remains unchanged.
\autoref{lst:brack}, shows our $n$-ary bracketing method to optimize the dual. 

\begin{lstlisting}[float, label={lst:primal}, language=Python, caption=\gls{troll}'s differentiable projection only calls a differentiable dual solver and otherwise uses standard autodiff operations.]
def TROLLProjection(log_target_prob, log_ref_prob, bound):
    kl_div = (log_target_prob.exp() * (log_target_prob - log_ref_prob)).sum(dim=-1)
    needs_projection = kl_div >= bound # only projects where necessary
    # ... masking of needed tokens
    # solve dual problem, i.e. find $\eta^*$
    opt_eta = DualSolver(log_target_prob, log_ref_prob, bound)
    primal_unnormalized = (opt_eta * log_ref_prob + log_target_prob) / (opt_eta + 1)
    primal = inner.log_softmax(dim=-1) 
    # ... combine masked unprojected and primal logits into one
    return projected_logits
\end{lstlisting}

\begin{lstlisting}[float, label={lst:dual}, language=Python, caption=Custom forward and backward code for \gls{troll}'s dual solver.]
def DualSolver.forward(log_target_prob, log_ref_prob, bound):
    # define objective in terms of log eta (such that eta > 0)
    opt_log_eta = optimize1d(
            lambda log_eta: dual(log_eta, ...),
            # ... bounds and termination config
        )
    # ... save for backward
    return opt_log_eta.exp()

def dual(log_eta, bound, log_target_prob, log_ref_prob)
    eta = log_eta.exp()
    inner = (log_target_prob + eta * log_ref_prob) / (eta + 1)
    inner_lse = logsumexp(inner, axis=-1)
    # negative of objective, since we minimize
    return eta *  bound + (eta + 1) * inner_lse

def DualSolver.backward(grad_output):
    # ... recompute primal = ... as in TROLLProjection
    one_plus_logratio = 1 + primal.log() - log_ref_prob
    # compute one_plus_logratio.T @ dprimal_dlog_output implicitly
    numerator = primal * (one_plus_logratio - vecdot(primal, one_plus_logratio).unsqueeze(-1) / (opt_eta + 1)
    # compute  dprimal_dopt_eta implicitly
    diff = log_ref_prob - log_target_prob
    dprimal_dopt_eta = primal * (diff - vecdot(primal, diff).unsqueeze(-1) / (opt_eta + 1)**2
    return grad_output * (numerator / -vecdot(one_plus_logratio, dprimal_dopt_eta))
\end{lstlisting}

\begin{lstlisting}[float, language=Python, label={lst:brack}, caption=N-ary Bracketing Search.]

class Optimizer1D:
      
    def batched_linspace(lower, upper, num_points):  
        # Batched linspace: lower and upper are (batch_size, 1), returns (batch_size, num_points)

        steps = linspace(0, 1, num_points)  
        return lower + (upper - lower) * steps

    def _opt_step(func, x, lower, upper):
        batch_size, num_points = x.shape
        # batched evaluation of all points
        y = func(x)
        # select min index for each batch element
        min_idx = argmin(y, dim=1)

        # take left and right point
        l_idx = min_idx - 1
        u_idx = min_idx + 1
        l_tmp = x[arange(batch_size), clamp(l_idx, 0, num_points - 1)]
        u_tmp = x[arange(batch_size), clamp(u_idx, 0, num_points - 1)]
        new_lower = where(l_idx < 0, lower), l_tmp)
        new_upper = where(u_idx >= num_points, upper, u_tmp)
        return new_lower, new_upper

    def optimize(func, lower, upper, num_points, max_steps, x_threshold):
        # batched, parallel, gradient-free, optimization of a 1D function 

        l, u = lower, upper
        # refine lower and upper until convergence
        for step in range(max_steps):
            x = Optimizer1D.batched_linspace(l, u, num_points + 2)
            x = x[:, 1:-1]

            l, u = Optimizer1D._opt_step(func, x, l, u)

            if ((l - u) < x_threshold).abs().all():
                break

        x = (l + u) / 2

        return x
\end{lstlisting}

\section{Experimental Setup}
\label{app_sec:setup}
\subsection{Models}

\begin{table}[t]
\centering
\begin{tabular}{lc}
\textbf{Model} & \textbf{Link} \\
\toprule
Qwen3-0.6B & \smolurl{https://huggingface.co/Qwen/Qwen3-0.6B} \\
Qwen3-1.7B & \smolurl{https://huggingface.co/Qwen/Qwen3-1.7B} \\
Qwen3-4B   & \smolurl{https://huggingface.co/Qwen/Qwen3-4B} \\
Qwen3-8B   & \smolurl{https://huggingface.co/Qwen/Qwen3-8B} \\
Qwen3-14B  & \smolurl{https://huggingface.co/Qwen/Qwen3-14B} \\
\midrule
Qwen2.5-0.5B-Instruct & \smolurl{https://huggingface.co/Qwen/Qwen2.5-0.5B-Instruct} \\
Qwen2.5-1.5B-Instruct & \smolurl{https://huggingface.co/Qwen/Qwen2.5-1.5B-Instruct} \\
Qwen2.5-3B-Instruct   & \smolurl{https://huggingface.co/Qwen/Qwen2.5-3B-Instruct} \\
Qwen2.5-7B-Instruct   & \smolurl{https://huggingface.co/Qwen/Qwen2.5-7B-Instruct} \\
\midrule
Llama-3.1.8B & \smolurl{https://huggingface.co/meta-llama/Llama-3.1-8B} \\
Llama-3.1.8B-Instruct  &  \smolurl{https://huggingface.co/meta-llama/Llama-3.1-8B-Instruct} \\ 
\midrule
Llama-3.2-3B  & \smolurl{https://huggingface.co/meta-llama/Llama-3.2-3B} \\ 
LLama-3.2-3B-Instruct  & \smolurl{https://huggingface.co/meta-llama/Llama-3.2-3B-Instruct} \\ 
FineMath-Llama 3B  &  \smolurl{https://huggingface.co/HuggingFaceTB/FineMath-Llama-3B} \\ 
\midrule
Apertus-8B  & \smolurl{https://huggingface.co/swiss-ai/Apertus-8B-2509} \\ 
Apertus-8B-Instruct  & \smolurl{https://huggingface.co/swiss-ai/Apertus-8B-Instruct-2509}  \\ 
\midrule
SmolLM3-3B  & \smolurl{https://huggingface.co/HuggingFaceTB/SmolLM3-3B} \\ 
\bottomrule
\end{tabular}
\caption{Model checkpoints used as starting points for finetuning throughout this work.}
\label{tab:all_models_ref}
\end{table}

\autoref{tab:all_models_ref} lists all model checkpoints used in this work.
They are publicly available and can be downloaded under the provided links.

We used the thinking mode for the models from the Qwen3-Family. 
For the non-instruct versions of Llama-3.1, Llama-3.2, and Apertus, we used the chat templates from the respective instruct versions. 

\subsection{Datasets}
\label{app_ssec:datasets}
\revision{
For all math datasets, sequence-level binary rewards are computed by parsing the \gls{llm} output through a regular expression, matching against a ground truth answer. 
For code generation, we evaluate each answer using the provided test cases. The reward is then computed as the fraction of successful tests.
If the execution of a test case exceeds a timeout of ten seconds, the evaluation is terminated and zero reward is given.}

\textbf{DAPO-Math.}
We build \texttt{DAPO} Train and \texttt{DAPO} Eval on the version of the DAPO-Math dataset provided by~\citet{cui2025entropy}\footnote{Their original datasets can be downloaded under \url{https://github.com/PRIME-RL/Entropy-Mechanism-of-RL}.} 
From their training set, we set aside $1024$ samples as an in-domain validation set (\texttt{DAPO} Eval), leaving $16{,}893$ samples for \texttt{DAPO} Train. 
For broader out-of-distribution evaluation, we again follow \citet{cui2025entropy} and use a benchmark suite, we refer to as \texttt{Math}-Eval, consisting of MATH500~\citep{hendrycks2021measuring}, AMC, AIME2024~\citep{li2024numinamath}, AIME 2025, OMNI-MATH~\citep{gao2025omni}, OlympiadBench~\citep{he2024olympiadbench}, and Minerva~\citep{lewkowycz2022solving}.
We again build the data provided by \citet{cui2025entropy} and also follow their protocol by computing the mean over 32 responses for the small but hard AMC, AIME2024, and AIME2025 datasets while only considering a single response for the other sets. 

Finally, we ensure all $3$ datasets have the same system preprompt, which we provide in \autoref{lst:prepromt}, and include correct and identical instructions for answer formatting. 

\begin{lstlisting}[float, caption={System Prompt for \texttt{DAPO}-Train, \texttt{DAPO}-Eval, and \texttt{Math}-Eval},label={lst:prepromt}, numbers=none, breaklines=true,breakautoindent=false,breakindent=0pt]
Your task is to follow a systematic, thorough reasoning process before providing the final solution. This involves analyzing, summarizing, exploring, reassessing, and refining your thought process through multiple iterations. Structure your response into two sections: Thought and Solution. In the Thought section, present your reasoning using the format: "<think> {thoughts} </think>".
\end{lstlisting}

\textbf{GSM8K.}
We use the publicly available train and validation sets of the \texttt{GSM8K} Dataset~\citep{hendrycks2021measuring}\footnote{\url{https://huggingface.co/datasets/openai/gsm8k}} without further modifications.

\textbf{Eurus-Math.}
We use the publicly available train and validation sets of the \texttt{Eurus}-2-RL-Dataset~\citep{cui2025processreinforcementimplicitrewards}\footnote{\url{https://huggingface.co/datasets/PRIME-RL/Eurus-2-RL-Data}},\revision{which is a subset of NuminaMath-CoT~\citep{li2024numinamath}}.
We filter for math questions, resulting in $455\,261$ train and $1\,024$ evaluation questions, \revision{and refer to the resulting dataset as \texttt{Eurus-Math}}.

\revision{
\textbf{Eurus-Code.}
We use the same \texttt{Eurus}-2-RL-Dataset~\citep{cui2025processreinforcementimplicitrewards} %
for code generation by filtering for code questions, resulting in $25\,276$ train and $1\,024$ evaluation questions.
We call this subset \texttt{Eurus-Code}.
Each includes multiple test cases of inputs and expected outputs after running the parsed python code within a sandbox.
We use the PRIME reward manager of \texttt{verl}\footnote{\url{https://github.com/volcengine/verl}} which evaluates up to the first ten test cases in the \texttt{SandboxFusion} sandbox\footnote{\url{https://github.com/bytedance/SandboxFusion}}.

The dataset consist of tasks from APPS~\citep{hendrycks2021measuring}, CodeContests~\citep{doi:10.1126/science.abq1158}, TACO~\citep{li2023taco}  and Codeforces\footnote{\url{https://huggingface.co/datasets/MatrixStudio/Codeforces-Python-Submissions}}.
For the evaluation, we compute the pass@1\ scores for each of the four benchmarks and average the results.
The validation split has 142, 377, 382 and 123 predefined questions for APPS, CodeContests, TACO, and Codeforces, respectively.
The train data is split 13.7\%, 38.1\%, 37.9\% and 10.3\%, respectively, such that the evaluation overweights APPS and Codeforces.
Empirically, the models seem to perform better on Codeforces but worse on TACO, such that the evaluation success rates are slightly higher.
}

\subsection{Training Setup}
\label{ssec:hyperparams}
\begin{table}[t]
   \centering
   \begin{tabular}{lcc}
   \textbf{Hyperparameter} & \textbf{Variable} & \textbf{Value} \\
   \toprule    
   Trust Region Size & $\epsilon$ & 0.05 \\
   KL Regression Factor & $\alpha$ & 1.0 \\
   Sparsity Remaining Mass & $1{-}\delta$ & $0.99999$ \\ 
   Max. Sparse Tokens & $K$ & $64$ \\ 
   Chunk Size & & $1024$ \\
   \midrule
   Clip Value & $\epsilon_\text{ppo}$ & $0.2$ \\
   \midrule
    Learning Rate & & $10^{-6}$ \\
    Gradient Max Norm & & $1.0$ \\
    Weight Decay & & $0.0$ \\
    Learning Rate-Schedule & & constant \\ 
    Learning Rate Critic (\gls{ppo} only) & & $10^{-5}$ \\
    Weight Decay Critic (\gls{ppo} only) & & $0.01$ \\
   \midrule
   Sampler Per Query & & $8$ \\
   Batch Size & & $32$  \\ 
   Batches Per Step & & $8$ \\     
   \end{tabular}
   \caption{Hyperparameters. We use these parameters for all experiments unless mentioned otherwise.}
   \label{tab:hp}
\end{table}

We provide hyperparameters for our training setup in \autoref{tab:hp}.
We maintain consistent hyperparameters across all experiments, except for \autoref{app_ssec:ablations}, where we always vary exactly one parameter.

\subsection{Hardware}
\label{ssec:hardware}

We train on clusters with Nvidia A100, H100, and H200 nodes, each equipped with 4 GPUs. 
For the Qwen3-$14$B, Qwen3-$8$B and Qwen2.5-$7$B-Instruct experiments in \autoref{ssec:results_qwen}, we use H200s.
For all other experiments, we use either H100 or A100 nodes, depending on model size.
We train most experiments for up to $2$ days, and extend some experiments on \texttt{DAPO} to up to $4$ days to show algorithm convergence. 
We always train \textit{Clip} and \textit{\gls{troll}} on identical hardware to ensure a fair comparison.

\section{Additional Results}
\label{app_sec:additional_results}

\subsection[Qwen on DAPO]{Qwen on \texttt{DAPO}}
\label{app_ssec:qwen_results}

\begin{figure}[t]

    \includegraphics[width=\linewidth]{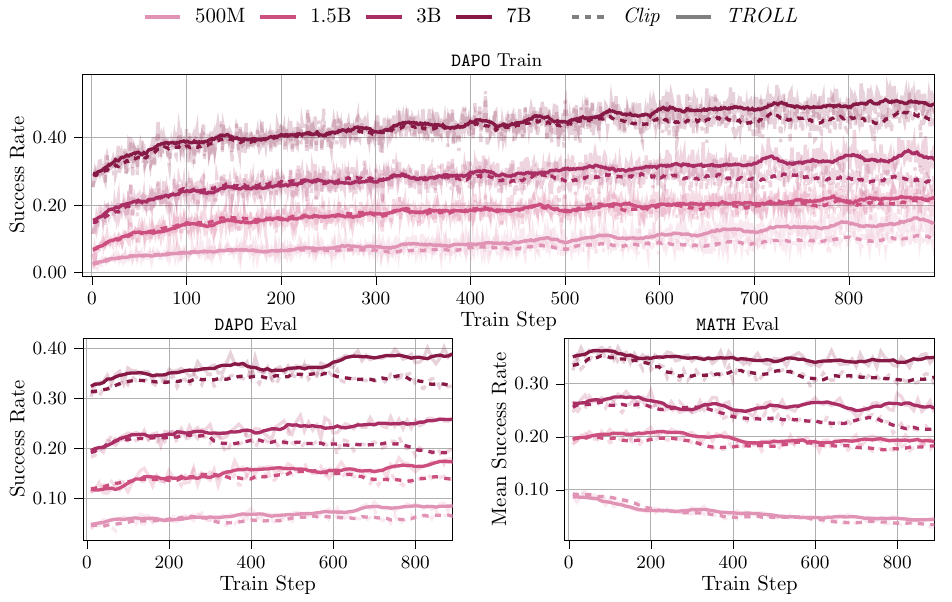}
    \caption{
        Performance of \textit{TROLL} and the \textit{Clip} objective across Qwen2.5-Instruct models with $500$M to $7$B parameters trained with GRPO on \texttt{DAPO}. 
        As in \autoref{fig:qwen3_size}, \textit{TROLL} yields more sample-efficient training and higher rewards at convergence. 
        These improvements extend both to evaluation on in-distribution questions and to generalization on out-of-distribution test datasets. 
        Smoothed values are shown in full opacity, with original curves in the background.  
    }
    \label{app_fig:qwen2_5_size}
\end{figure}

\autoref{app_fig:qwen2_5_size} extends the setup of the math evaluation of \autoref{fig:qwen3_size} to Qwen2.5-Instruct models.
Similarly to the Qwen3 results, \textit{\gls{troll}} consistently improves over the \textit{Clip} objective for each model size.
We further find that, generally, most Qwen2.5 models slightly overfit on the training data, although this effect is less pronounced for \textit{\gls{troll}}.

\autoref{app_fig:qwen3_methods} and \autoref{app_fig:qwen2_5_methods} show complete training and evaluation curves for \autoref{tab:qwen_methods}.
We find that \textit{\gls{troll}} improves training success rates over \textit{Clip} for both models and across methods, to the point where Qwen3 \gls{grpo} and \gls{drgrpo} start to slightly overfit on the out-of-distribution \texttt{MATH} evaluation.
Interestingly, while \textit{Clip} leads to unstable performance and eventual divergence for \gls{gspo} for both Qwen2.5 and Qwen3, \textit{\gls{troll}}'s token-level trust region optimization remains stable.

\begin{figure}[t]
    
    \includegraphics[width=\linewidth]{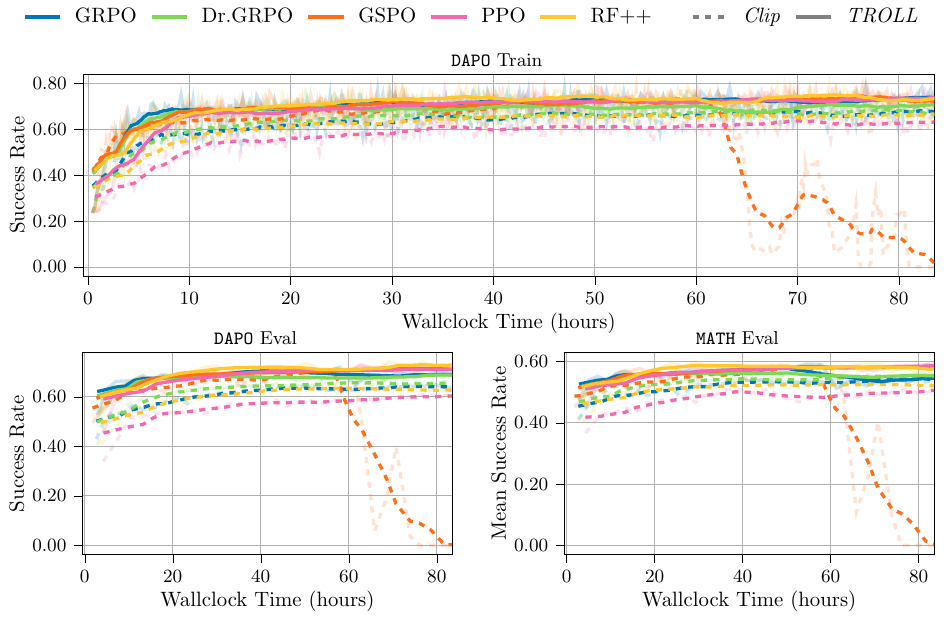}
    \caption{
        \textit{\gls{troll}} and \textit{Clip} success rates for Qwen3-$8$B trained with \gls{grpo}, \gls{drgrpo}, \gls{gspo}, \gls{ppo} \revision{and \gls{rfpp}} on training data (\textbf{top}), 
        in-domain evaluation (\textbf{bottom left}) and out-of-domain evaluation (\textbf{bottom right}).
        Smoothed values are shown in full opacity, with original curves in the background.  
        \textit{TROLL} improves over the \textit{Clip} objective for all methods. 
        For \gls{gspo}, \textit{Clip} eventually diverges, leading to $0.00\%$ success rate on all metrics, while \textit{TROLL}'s optimization stays stable. 
    }
    \label{app_fig:qwen3_methods}
\end{figure}

\begin{figure}[t]

    \includegraphics[width=\linewidth]{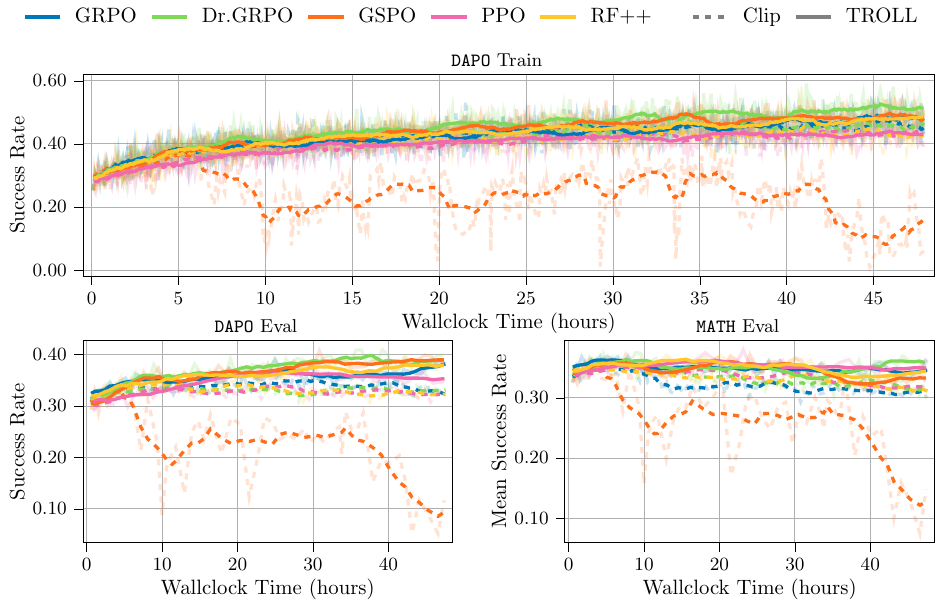}
    \caption{
        \textit{TROLL} and \textit{Clip} success rates across Qwen2.5-$7$B-Instruct models trained with \gls{grpo}, \gls{drgrpo}, \gls{gspo}, \gls{ppo} \revision{and \gls{rfpp}} on training data (\textbf{top}), 
        in-domain evaluation (\textbf{bottom left}) and out-of-domain evaluation (\textbf{bottom right}).
        Smoothed values are shown in full opacity, with original curves in the background.  
        \textit{TROLL} improves over the \textit{Clip} objective for all methods. 
        For \gls{gspo}, \textit{Clip} eventually diverges, while \textit{TROLL}'s optimization stays stable. 
    }
    \label{app_fig:qwen2_5_methods}
\end{figure}

\begin{table}
    \centering
    \small{
    \begin{tabular}{llccccccc}
    \toprule
    Method &  & AIME24 & AIME25 & AMC & MATH & Omni-Math & Olympiad & Minerva \\
    \midrule
    \multicolumn{9}{c}{Qwen2.5-$7$B-Instruct} \\
    \midrule
    GRPO   & \textit{Clip} & 0.066 & 0.075 & 0.535 & 0.683 & 0.239 & 0.286 & 0.304 \\
           & \textit{\gls{troll}} & 0.168 & 0.129 & 0.587 & 0.712 & 0.254 & 0.317 & 0.284 \\
    \midrule
    Dr.GRPO & \textit{Clip}  & 0.103 & 0.067 & 0.560 & 0.662 & 0.242 & 0.288 & 0.295 \\
           & \textit{\gls{troll}} & 0.168 & 0.135 & 0.605 & 0.706 & 0.259 & 0.320 & 0.317 \\
    \midrule
    PPO    & \textit{Clip}  & 0.092 & 0.064 & 0.503 & 0.706 & 0.251 & 0.316 & 0.299 \\
           & \textit{\gls{troll}} & 0.162 & 0.093 & 0.547 & 0.734 & 0.258 & 0.320 & 0.332 \\
    \midrule
    GSPO   & \textit{Clip}  & 0.026 & 0.002 & 0.188 & 0.344 & 0.106 & 0.102 & 0.120 \\
           & \textit{\gls{troll}} & 0.159 & 0.076 & 0.531 & 0.699 & 0.257 & 0.297 & 0.310 \\
    \midrule
    \revision{RF++} & \revision{\textit{Clip}} & \revision{0.096} & \revision{0.033} & \revision{0.556} & \revision{0.666} & \revision{0.239} & \revision{0.291} & \revision{0.299} \\
    & \revision{\textit{\gls{troll}}} & \revision{0.166} & \revision{0.102} & \revision{0.552} & \revision{0.687} & \revision{0.263} & \revision{0.320} & \revision{0.319} \\
    \midrule
    \multicolumn{9}{c}{Qwen3-$8$B} \\
    \midrule
    GRPO   & \textit{Clip}  & 0.439 & 0.293 & 0.720 & 0.889 & 0.465 & 0.547 & 0.431 \\
           & \textit{\gls{troll}} & 0.547 & 0.353 & 0.790 & 0.812 & 0.465 & 0.497 & 0.391 \\
    \midrule
    Dr.GRPO & \textit{Clip}  & 0.458 & 0.305 & 0.743 & 0.891 & 0.477 & 0.547 & 0.425 \\
           & \textit{\gls{troll}} & 0.447 & 0.337 & 0.769 & 0.880 & 0.466 & 0.522 & 0.403 \\
    \midrule
    PPO    & \textit{Clip}  & 0.380 & 0.234 & 0.694 & 0.874 & 0.439 & 0.531 & 0.405 \\
           & \textit{\gls{troll}} & 0.524 & 0.408 & 0.780 & 0.910 & 0.521 & 0.567 & 0.425 \\
    \midrule
    GSPO   & \textit{Clip}  & 0.000 & 0.000 & 0.000 & 0.000 & 0.000 & 0.000 & 0.000 \\
           & \textit{\gls{troll}} & 0.474 & 0.407 & 0.813 & 0.897 & 0.514 & 0.547 & 0.408 \\
    \midrule
    \revision{RF++} & \revision{\textit{Clip}} & \revision{0.356} & \revision{0.251} & \revision{0.728} & \revision{0.897} & \revision{0.459} & \revision{0.538} & \revision{0.415} \\
    & \revision{\textit{\gls{troll}}} & \revision{0.490} & \revision{0.355} & \revision{0.803} & \revision{0.891} & \revision{0.528} & \revision{0.563} & \revision{0.415} \\
    \bottomrule
    \end{tabular}
    }
    \caption{Success rates for individual \texttt{MATH} test datasets for Qwen2.5-$7$B-Instruct and Qwen3-$8$B models trained on \texttt{DAPO} with different advantage estimation methods.
    \textit{\gls{troll}} provides consistent benefits across methods and evaluation tasks, showing well-balanced improvements in performance.
    It also successfully trains \gls{gspo} without divergence, wheres \textit{Clip} eventually causes unstable updates, as shown in \autoref{app_fig:qwen3_methods} and \autoref{app_fig:qwen2_5_methods}.
    }
    \label{app_tab:qwen_methods_tasks}
\end{table}

\subsection{Additional Models}
\label{app_ssec:more_models}

\begin{figure}[t]
    \centering
    
    \includegraphics{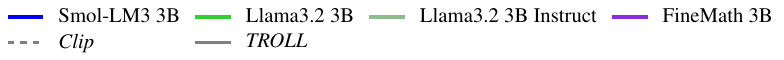} \\[0.5em]
    
    \begin{minipage}[t]{0.49\textwidth}
        \centering
        \includegraphics{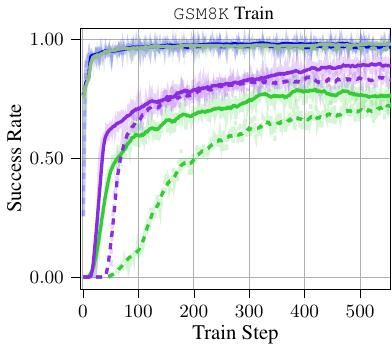}
    \end{minipage}
    \hfill
    \begin{minipage}[t]{0.49\textwidth}
        \centering
        \includegraphics{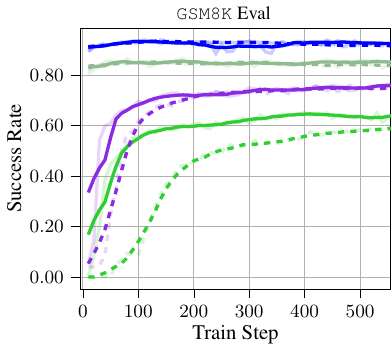}
    \end{minipage}
    \caption{
        \textit{\gls{troll}} and \textit{Clip} success rates for different $3$B models trained with GRPO on the \texttt{GSM8K} training data (\textbf{left}) and evaluated on the \texttt{GSM8K} test set (\textbf{right}).
        Smoothed values are shown in full opacity, with original curves in the background.  
        \textit{\gls{troll}} generally causes models to pick up a training signal more quickly, and exhibits more stable training behavior.
    }
    \label{app_fig:gsm8k_others_mid}
\end{figure}

\begin{figure}[t]
    \centering
    
    \includegraphics{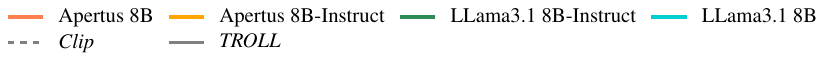} \\[0.5em]
    
    \begin{minipage}[t]{0.49\textwidth}
        \centering
        \includegraphics{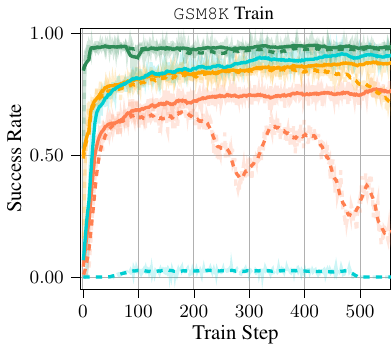}
    \end{minipage}
    \hfill
    \begin{minipage}[t]{0.49\textwidth}
        \centering
        \includegraphics{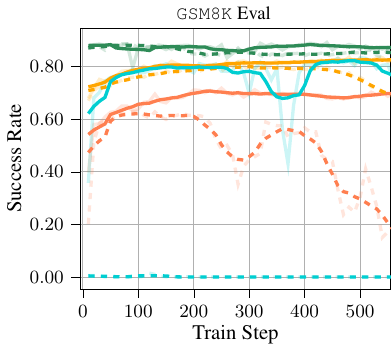}
    \end{minipage}
    \caption{
        \textit{\gls{troll}} and \textit{Clip} success rates for different $8$B models trained with GRPO on the \texttt{GSM8K} training data (\textbf{left}) and evaluated on the \texttt{GSM8K} test set (\textbf{right}).
        Smoothed values are shown in full opacity, with original curves in the background.  
        \textit{\gls{troll}} generally causes models to pick up a training signal more quickly, and exhibits more stable training behavior.
    }
    \label{app_fig:gsm8k_others_large}
\end{figure}

\autoref{app_fig:gsm8k_others_mid} and \autoref{app_fig:gsm8k_others_large} show success rates for different $3$B and $8$B models, respectively.
We find that \textit{\gls{troll}} causes some models, such as Finemath-$3$B, Llama3.2-$3$B and Llama3.1-$8$B to receive a training signal in significantly fewer steps.
Other models, such as Apertus-$8$B show more stable performance when trained with \textit{\gls{troll}}.
Finally, for models that work well with the \textit{Clip} objective, using \textit{\gls{troll}} generally yields some performance benefit even though the success rates on \texttt{GSM8K} are almost saturated.

\subsection[Qwen3 on Eurus-Math and GSM8K]{Qwen3 on \revision{\texttt{Eurus-Math}} and \texttt{GSM8K}}
\label{app_ssec:dataset_results}

We additionally evaluate different Qwen3 model sizes on \texttt{GSM8K} in \autoref{app_fig:gsm8k_qwen3}, finding that most models quickly saturate on this comparatively easy task.
Nevertheless, using \textit{\gls{troll}} instead of \textit{Clip} generally provides a small boost in performance across model sizes.
Similarly, \autoref{app_fig:eurus_qwen3_8b} shows that \textit{\gls{troll}} leads to improvements for Qwen3-$8$B trained with \gls{grpo} on \revision{\texttt{Eurus-Math}}.

\begin{figure}[t]
    \centering
    \includegraphics{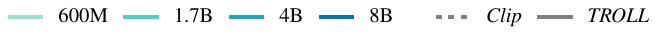} \\[0.5em]
    
    \begin{minipage}[t]{0.49\textwidth}
        \centering
        \includegraphics{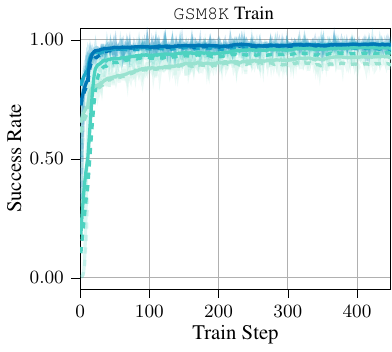}
    \end{minipage}
    \hfill
    \begin{minipage}[t]{0.49\textwidth}
        \centering
        \includegraphics{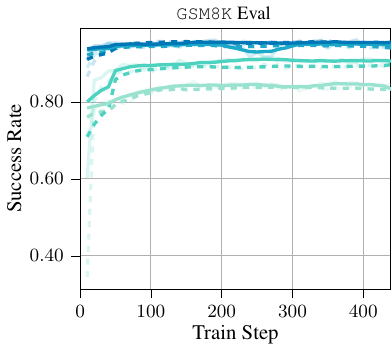}
    \end{minipage}
    \caption{
        \textit{TROLL} and \textit{Clip} success rates for Qwen3 models with $600$M to $8$B parameters trained with GRPO on the \texttt{GSM8K} training data (\textbf{left}) and evaluated on the \texttt{GSM8K} test set (\textbf{right}).
        Smoothed values are shown in full opacity, with original curves in the background.  
        Both \textit{\gls{troll}} and \textit{Clip} quickly converge in all cases, although \textit{\gls{troll}} achieves slightly higher performance for most model sizes.
    }
    \label{app_fig:gsm8k_qwen3}
\end{figure}

\begin{figure}[t]
    \centering
    
    \includegraphics{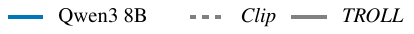} \\[0.5em]
    
    \begin{minipage}[t]{0.49\textwidth}
        \centering
        \includegraphics{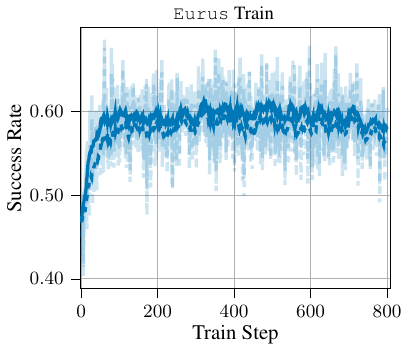}
    \end{minipage}
    \hfill
    \begin{minipage}[t]{0.49\textwidth}
        \centering
        \includegraphics{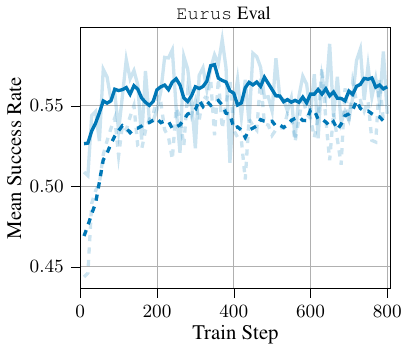}
    \end{minipage}
    \caption{
        \textit{\gls{troll}} and \textit{Clip} success rates for Qwen3-$8$B trained with GRPO on the \texttt{Eurus-\revision{Math}} training data (\textbf{left}) and evaluated on the \texttt{Eurus-\revision{Math}} test set (\textbf{right}).
        Smoothed values are shown in full opacity, with original curves in the background.  
        \textit{\gls{troll}} performs slightly but consistently better during training, and generalizes well to the test set.
    }
    \label{app_fig:eurus_qwen3_8b}
\end{figure}

\revision{\subsection{Additional Training Algorithms}
\label{app_ssec:additional_training_algos}
Finally, we evaluate Qwen3-$8$B using BAPO~\citep{xi2025bapo} as an adaptive clipping heuristic that acts as an alternative to regular clipping, and GPG~\citep{chu2025gpg}, which is a clipping-free policy gradient method.
\autoref{app_fig:qwen3_baselines} finds that BAPO's adaptive clipping does not significantly improve over regular clipping in our case, and still performs worse than \gls{troll}.

Although clipping-free GPG initially learns well, it eventually diverges during training, likely due to unstable policy updates.
The figure illustrates how \gls{troll} can benefit such non-clipping-based policy gradient methods. 
Unlike \gls{ppo}-style clipping, which relies on a ratio formulation, \gls{troll} limits updates by enforcing trust regions directly on the LLM output.
Thus, \gls{troll} can in principle be applied to any LLM post-training algorithm, including GPG. 
In our experiments, this approach effectively resolves GPG's stability issues, as demonstrated by the GPG(\textit{TROLL}) variant.}

\begin{figure}[t]
    \includegraphics[width=\linewidth]{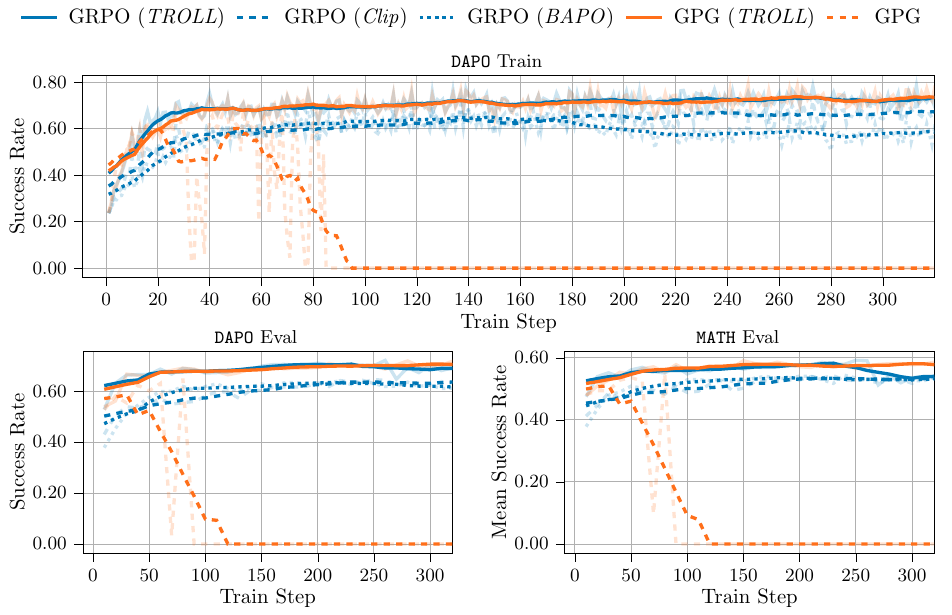}
    \caption{
        \revision{Success rates for Qwen3-$8$B on \texttt{DAPO} for training data (\textbf{top}), 
        in-domain evaluation (\textbf{bottom left}) and out-of-domain evaluation (\textbf{bottom right}).
        Smoothed values are shown in full opacity, with original curves in the background. 
        \textit{\gls{troll}} performs well when combined with either GRPO or the clipping-free GPG baseline, while vanilla GPG eventually diverges. 
        BAPO's adaptive clipping slightly improves over regular \textit{Clip} on the evaluation datasets, but both underperform \textit{\gls{troll}} by a fair margin.}
    }
    \label{app_fig:qwen3_baselines}
\end{figure}

\section{Further Analysis}

\subsection{Projection Parameter Sensitivity}
\label{app_ssec:ablations}

\begin{figure}[t]
    
    \includegraphics[width=\linewidth]{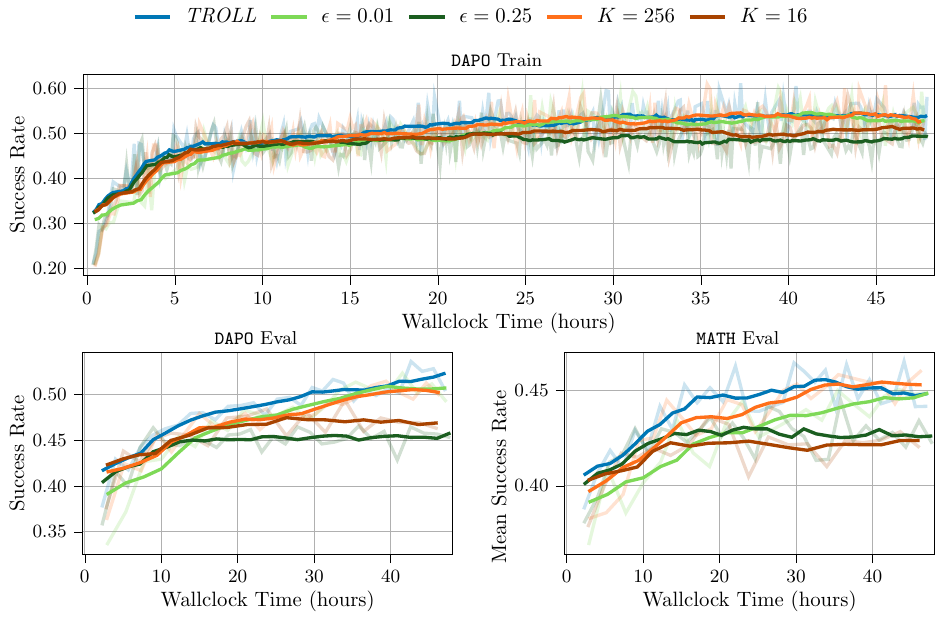}
    \caption{
        Qwen3-$1.7$B trained with GRPO using the \textit{\gls{troll}} projection compared to different hyperparameter choices
        on \texttt{DAPO} training data (\textbf{top}), in-domain evaluation (\textbf{bottom left}) and out-of-domain evaluation (\textbf{bottom right}).
        Smoothed values are shown in full opacity, with original curves in the background.  
        \textit{\gls{troll}} works well for reasonable \gls{kl} bounds $\epsilon$ and top-$K$ logit selections, but slightly degrades for too-large bounds and too few kept logits. 
    }
    \label{app_fig:ablations}
\end{figure}

We experiment with different \gls{kl} bounds, testing $\epsilon{=}0.01$ and $\epsilon{=}0.25$ instead of the default $\epsilon{=}0.05$.
Additionally, we try different levels of sparsification.
We switch the maximum number of kept tokens from $K{=}64$ to a lower $K{=}16$ and a higher $K{=}256$, adjusting the distribution mass threshold $\delta$ from $1e{-}5$ to $1e{-}4$ and $1e{-}6$ accordingly.
\autoref{app_fig:ablations} shows that a lower \gls{kl} bound $\epsilon$ for the projection leads to slower learning, but eventually reaches comparable performance. 
In contrast, a higher \gls{kl} bound leads to worse performance, presumably because the policy moves too quickly during update steps.
Reducing the number of kept tokens leads to worse overall performance, which is likely caused by incorrect \gls{kl} estimates and thus sub-optimal projections.
A higher amount of kept tokens does not yield any additional benefit, however, suggesting that $K{=}64$ and $\epsilon{=}1e$-$5$ maintain a sufficiently close approximation of the real policy logit distributions. 

\subsection[Batch Size]{\revision{Batch Size}}
\label{app_ssec:batch_size}
\revision{\textit{\gls{troll}}'s trust region projection promises more stable policy updates by constraining the difference between the new policy and the policy used to generate the data, in our case the responses. 
To evaluate this stability, we experiment with different batch sizes for Qwen3-$1.7$B trained with GRPO on \texttt{DAPO}, leaving all other parameters untouched.
Here, doubling the batch size doubles the number of gradient steps between data collection, thus increasing the difference between the old policy $\oldpol$ and the current policy $tarpol$ over the course of the training step.
\autoref{app_fig:batch_size} shows that \textit{\gls{troll}} remains stable when increasing the batch size from our default of $256$ to $2\,048$, indicating that \textit{\gls{troll}} can successfully train on older data due to its projection. 
In comparison, \textit{Clip} shows a gradual degradation in performance for each doubling of the batch size, likely because the clipping does not address a potential divergence of the current policy to the policy that was used to generate the data.
}

\begin{figure}[t]
    \includegraphics[width=\linewidth]{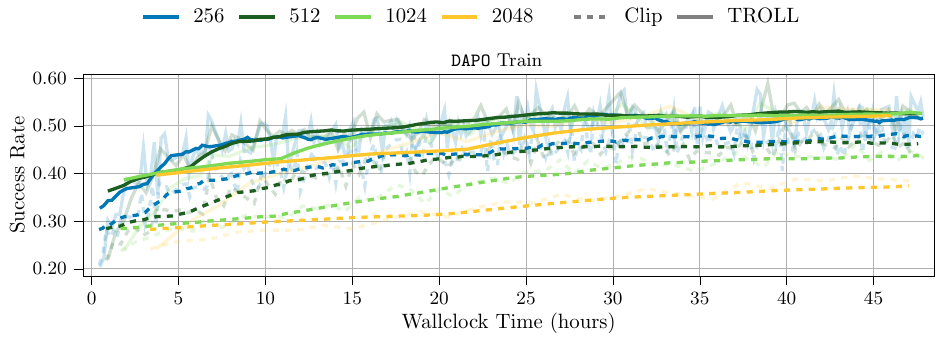}
    \caption{
        \revision{\textit{\gls{troll}} and \textit{Clip} training success rates for Qwen3-$1.7$B trained with GRPO on \texttt{DAPO} for different batch sizes.
        \textit{\gls{troll}} remains stable for larger batch sizes, while \textit{Clip} gradually and consistently degrades when increasing the batch size.}
    }
    \label{app_fig:batch_size}
\end{figure}

\subsection{Generated Sequence Analysis}
\label{app_ssec:analysis}
We analyze the behavior of different Qwen3 and Qwen2.5 model sizes trained on \texttt{DAPO} using \gls{grpo} in \autoref{app_fig:analysis}.
There is a general trend that larger models require fewer selected tokens to satisfy the sparsity mass threshold of $\delta{=}1e{-}5$, which is consistent with established \gls{llm} scaling laws~\citep{kaplan2020scaling}.
Here, as little as $5{-}10$ tokens are sufficient to capture most of the mass for the larger models. 
For the larger Qwen3 models, this trend appears less pronounced, likely because these models are to some extent saturating the \texttt{DAPO} benchmark. 
We also observe clear differences in response length dynamics over training. 
\gls{troll} generally adapts the token length much faster than \textit{Clip}.
\gls{troll} reduces the response length for Qwen3, while increasing it for Qwen2.5-Instruct.
This difference originates in the different behavior of the pretrained models used to initialize learning, as the Qwen3 models tend to generate much longer responses, presumably due to their built-in thinking mode.  
After the RL fine-tuning with \gls{troll}, the response lengths of both model families are more similar.
In contrast, models trained with \textit{Clip} show much slower shifts in response length. 
This quicker adjustment under \gls{troll} aligns with the faster performance gains observed in both model families.
Finally, both approaches clip or project slightly more than $0.1\%$ of tokens for most of the training, but \textit{\gls{troll}}'s projection exhibits a lot more variance and tends to increase in later training stages, potentially suggesting a more active involvement in the learning process. 
In some cases \textit{\gls{troll}} projects a lot more aggressively, although this increase in projections does not cause a degradation in model performance.

\begin{figure}[t]
    \centering
    
    \begin{tikzpicture}
\definecolor{darkcyan0119182}{RGB}{0,119,182}
\definecolor{darkgray176}{RGB}{176,176,176}
\definecolor{lightblue153226208}{RGB}{153,226,208}
\definecolor{lightgray204}{RGB}{204,204,204}
\definecolor{lightseagreen28169201}{RGB}{28,169,201}
\definecolor{mediumturquoise77210192}{RGB}{77,210,192}
\definecolor{midnightblue06392}{RGB}{0,63,92}
\begin{axis}[
        hide axis,
        xmin=0,
        xmax=1,
        ymin=0,
        ymax=1,
        legend columns=8,
        legend cell align=left,
        font=\small,
        legend style={
            draw=none,
            column sep=1ex,
            line width=1pt
        }
    ]
\addlegendimage{ultra thick, lightblue153226208}
\addlegendentry{600M}
\addlegendimage{ultra thick, mediumturquoise77210192}
\addlegendentry{1.7B}
\addlegendimage{ultra thick, lightseagreen28169201}
\addlegendentry{4B}
\addlegendimage{ultra thick, darkcyan0119182}
\addlegendentry{8B}
\addlegendimage{empty legend}
\addlegendentry{~}
\addlegendimage{ultra thick, gray, dashed}
\addlegendentry{\textit{Clip}}
\addlegendimage{ultra thick, gray}
\addlegendentry{\textit{TROLL}}
\end{axis}
\end{tikzpicture} \\[0.5em]
    
    \input{fig/result_plots/eurus_code_qwen3_model_size/grouped_vs_ent}
    \caption{
        \revision{
        \textit{\gls{troll}} and \textit{Clip} entropy (\textbf{left}) and success rate (\textbf{right}) for different Qwen3 models trained with GRPO on the \texttt{Eurus-Code} training data.
        Smoothed values are shown in full opacity, with original curves in the background.  
        \textit{\gls{troll}} generally causes less decrease in token entropy while \textit{Clip} shows a strong negative correlation between success rate and entropy.
        The quick improvement of Qwen3-8B \textit{Clip} around step 40 coincides with a rapid drop in entropy.
        }
    }
    \label{app_fig:eurus_code_ent}
\end{figure}

\begin{figure}[t]
   \centering
    \includegraphics[width=\textwidth]{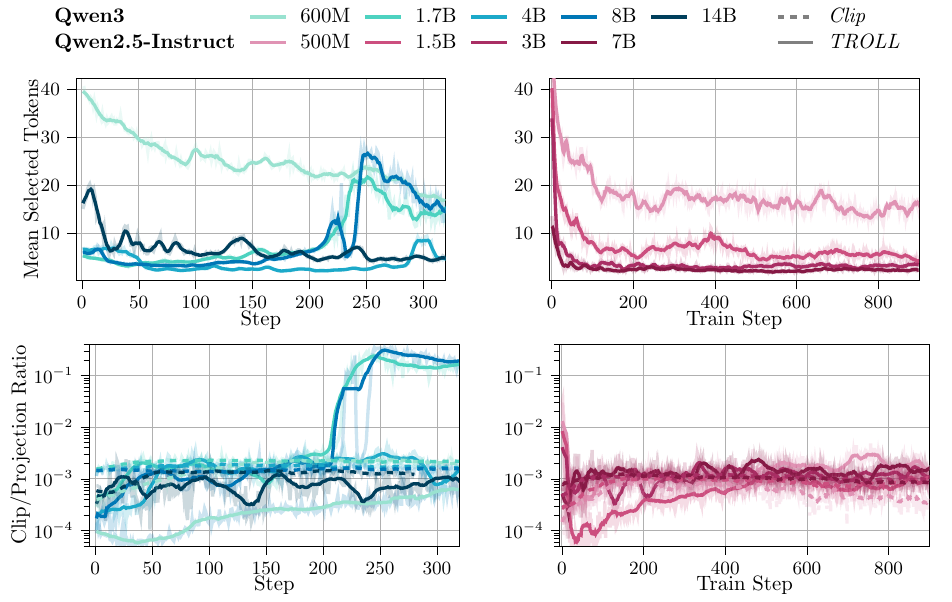}%
    \\
    \includegraphics[width=\textwidth]{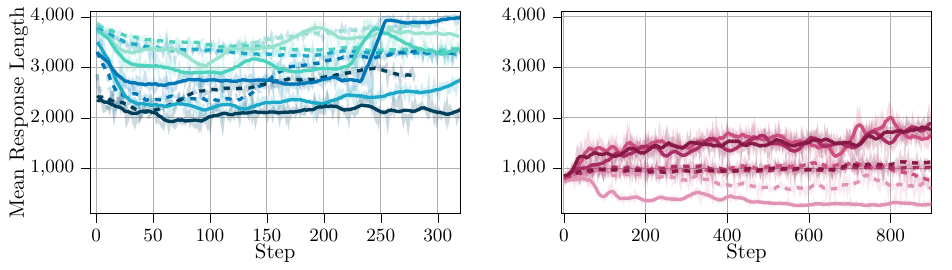}%

    \caption{
    Training dynamics of Qwen3 models on \texttt{DAPO} using GRPO.
    Smoothed values are shown in full opacity, with original curves in the background.  
    Larger models need fewer tokens to meet the sparsity threshold of $\delta{=}10^{-5}$. While both approaches affect ${\sim}0.1$\% of tokens most of the time, \textit{\gls{troll}} tends to increase projection later during training without harming performance.
    \textit{\gls{troll}} quickly adjusts response lengths while achieving higher success rates, whereas \textit{Clip} is slower to alter the response length over time. 
    }
    \label{app_fig:analysis}
\end{figure}

\subsection{TROLL Overhead}
\label{app_ssec:troll_overhead}
From a practical perspective, the number of selected tokens by the sparsification and the computational and memory overhead of \gls{troll} are relevant.
Fair comparison of the computation and memory overhead are tricky, as both the response length and number of kept tokens varies during training and models.
Initially, the average number of logits needed to achieve the desired total mass changes quickly, and small models tend to require significantly more logits due to their higher perplexity.

We therefore evaluate at the initial state by preventing updates with tight KL and clip ratio bounds.
Note that setting a low, i.e.\ zero, learning rate is not representative since the LLM outputs must change to activate the trust region projection.
\autoref{app_tab:troll_overhead} compares the memory and runtime overhead of using \gls{troll} and \gls{grpo} with small Qwen3 models on \texttt{GSM8K} trained on 4x Nvidia A100-40GB GPUs.
\emph{Chunked} refers to another simple memory reduction trick, where we normalize and sparsify only a chunk, in this case size 1024, at once and avoid the dense single precision upcast of the entire mini batch.
To compare the different model sizes with different response length characteristics fairly, we clip all answers to just $256$ tokens.
Note that this length is just short enough that most answers are clipped, while some prompts are still solved.
Then all models have almost the maximal $256$ response tokens \emph{on average} (about $255.4$), yet still have an update gradient to reach the trust region boundary.

The responses for a single prompt with a \gls{grpo} group size of 8 and this significantly simplified sequence length of $256$, float-32 representations for the probabilities, and Qwen3's tokenizer with a vocabulary of $151\,936$ tokens yields a memory overhead of
\begin{equation}
    \label{app_eq:naive_sparse_overhead}
    256\cdot8\cdot 151936 \cdot 4\mathrm{B} \approx 1.16 \mathrm{GiB}
\end{equation}
for the dense implementation.
Sparsification instead requires an average of $5{-}10$ logits per token (\autoref{app_fig:analysis} top), reducing the memory to less than $1\mathrm{MiB}$.
For each iteration, all methods need to store a rollout buffer of answers, in our case of size $256$.
In addition, the current mini-batch for the policy update needs to be stored.
While this overhead can be reduced to a single answer with gradient accumulation, the rollout buffer still needs to store all outputs of the old policy. 
The total memory overhead reported in \autoref{app_tab:troll_overhead} with 256 prompts is only ${\approx}6.3$ GiB across the four GPUs while just the dense distribution storage would already require ${\approx}296$ GiB of memory, showing the necessity of our sparsification.

\begin{table}[t]
\centering
\small{
    \setlength\tabcolsep{3pt}
    \begin{tabular}{lccc}
    \textbf{Metric (Method)} & \textbf{Qwen3-0.6B} & \textbf{Qwen3-1.7B} & \textbf{Qwen3-4B} \\
    \toprule
    VRAM (\textit{Clip}) & 25.415 GiB & 28.418 GiB & 34.574 GiB \\
    VRAM (\textit{\gls{troll}}) & 27.868 GiB & 30.663 GiB & 36.837 GiB \\
    VRAM (\textit{\gls{troll}} Chunked)  & 27.227 GiB & 29.994 GiB & 36.157 GiB \\
    VRAM Delta (Chunked) & +1.812 GiB (+7.1\%) & +1.576 GiB (+5.5\%) & +1.583 GiB (+4.6\%) \\
    \midrule
    Runtime (\textit{Clip}) & 30.874 s & 43.372 s & 85.133 s \\
    Runtime (\textit{\gls{troll}}) & 46.715 s & 49.053 s & 90.570 s \\
    Runtime (\textit{\gls{troll}} Chunked)  & 47.600 s & 50.629 s & 92.906 s \\
    Runtime Delta (Chunked) & +16.726 s (+54.2\%) & +7.257 s (+16.7\%) & +7.773 s (+9.1\%) \\
    \bottomrule
    \vspace{-0.2cm}
    \end{tabular}
    }
    \caption{Max allocated VRAM and runtime of one iteration. The smallest 0.6B models does not fully saturate the GPU, so the Delta results differ from the larger models. The projection overhead is independent of the model size and already below ten percent for the small 4B model and slower A100 GPU. The advantage of the chunked sparsification depends on the micro batch size, so the benefit is larger for bigger GPUs.}
    \label{app_tab:troll_overhead}
    \vspace{-0.3cm}
\end{table}

\clearpage
\section{Example Generations}
\label{app_sec:example_generations}

\pagebreak[0]
\begin{tcolorbox}[
    title=Prompt 1,
    lower separated=false,
    colback=white!95!black,
    colframe=white!90!black,
    fonttitle=\bfseries,
    coltitle=black,
    enhanced,
    boxed title style={
        colback=white!85!black,
        colframe=white!80!black,
        fonttitle=black
    },
    attach boxed title to top left={xshift=0.5cm,yshift=-2mm}
    ]
\begin{lstlisting}[language={}, frame=none]
system
Your task is to follow a systematic, thorough reasoning process before providing the final solution. This involves analyzing, summarizing, exploring, reassessing, and refining your thought process through multiple iterations. Structure your response into two sections: Thought and Solution. In the Thought section, present your reasoning using the format: "<think>
 {thoughts} </think>
". Each thought should include detailed analysis, brainstorming, verification, and refinement of ideas. After "</think>
," in the Solution section, provide the final, logical, and accurate answer, clearly derived from the exploration in the Thought section.
user
A list of positive integers has the following properties:
$\bullet$ The sum of the items in the list is $30$.
$\bullet$ The unique mode of the list is $9$.
$\bullet$ The median of the list is a positive integer that does not appear in the list itself.
 Find the sum of the squares of all the items in the list.
Present the answer in LaTeX format: \boxed{Your answer}.
assistant
\end{lstlisting}
\end{tcolorbox}

\begin{tcbinputlisting}{
  enhanced,
  breakable,
  listing engine=listings,
  listing file=appendix/val_gens/14b_dapo/trpl_out2.tex,
  text only,
  title=\gls{troll} 14B Response,
    lower separated=false,
    colback=white!95!midnightblue06392,
    colframe=white!90!midnightblue06392,
    fonttitle=\bfseries,
    coltitle=black,
    boxed title style={
        colback=white!80!midnightblue06392,
        colframe=white!00!midnightblue06392,
        fonttitle=black
    },
    attach boxed title to top left={xshift=0.5cm,yshift=-2mm},
}
\end{tcbinputlisting}

\begin{tcbinputlisting}{
  enhanced,
  breakable,
  listing engine=listings,
  listing file=appendix/val_gens/14b_dapo/vanilla_out2.tex,
  text only,
  title=Clip 14B Response,
    lower separated=false,
    colback=white!95!red,
    colframe=white!90!red,
    fonttitle=\bfseries,
    coltitle=black,
    boxed title style={
        colback=white!80!red,
        colframe=white!00!red,
        fonttitle=black
    },
    attach boxed title to top left={xshift=0.5cm,yshift=-2mm},
}
\end{tcbinputlisting}

\clearpage

\end{document}